\documentclass{article}

\usepackage[preprint]{neurips_2026}

\usepackage[utf8]{inputenc} 
\usepackage[T1]{fontenc}    
\usepackage{hyperref}       
\usepackage{url}            
\usepackage{booktabs}       
\usepackage{amsfonts}       
\usepackage{nicefrac}       
\usepackage{microtype}      
\usepackage{xcolor}         
\usepackage{makecell}
\usepackage{graphicx}
\usepackage{wrapfig}
\usepackage{subcaption}
\usepackage[font=small]{caption}   

\usepackage{amsmath,amsfonts,bm}
\usepackage{amssymb}

\def\eqref#1{equation~\ref{#1}}

\def\1{\bm{1}}

\def\mI{{\bm{I}}}

\DeclareMathAlphabet{\mathsfit}{\encodingdefault}{\sfdefault}{m}{sl}
\SetMathAlphabet{\mathsfit}{bold}{\encodingdefault}{\sfdefault}{bx}{n}

\newcommand{\R}{\mathbb{R}}

\usepackage{amsthm}

\newtheorem{thm}{Theorem}[section]
\newtheorem{lem}[thm]{Lemma}
\newtheorem{cor}[thm]{Corollary}

\theoremstyle{definition}

\theoremstyle{remark}

\title{A Mutual Information Lower Bound for Multimodal Regression Active Learning}

\author{%
  Leonardo Ferreira Guilhoto\thanks{Equal contribution.} \\
  Graduate Group on Applied Mathematics \& Computational Science \\
  Philadelphia, PA 19104 \\
  \texttt{guilhoto@sas.upenn.edu} \\
  \AND
  Akshat Kaushal\footnotemark[1] \\
  Department of Computer and Information Science\\
  University of Pennsylvania \\
  Philadelphia, PA 19104\\
  \texttt{akaush@seas.upenn.edu} \\
  \AND
  Paris Perdikaris \\
  Department of Mechanical Engineering and Applied Mechanics\\
  University of Pennsylvania \\
  Philadelphia, PA 19104\\
  \texttt{pgp@seas.upenn.edu} \\
}

\begin{document}

\maketitle

\begin{abstract}
    Active learning for continuous regression has lacked an acquisition function that targets epistemic uncertainty when the predictive distribution is multimodal: variance misses modal disagreement, and information-theoretic targets like BALD are designed for discrete outputs. We introduce a Two-Index framework that makes this separation explicit: one stochastic index selects among competing model hypotheses (epistemic source), while a second governs within-hypothesis randomness (aleatoric source). An entropy decomposition within the framework identifies the mutual information between the output and the epistemic index as a principled acquisition objective, and we prove this quantity vanishes as the model is trained on growing datasets, confirming that it captures exactly the uncertainty data can resolve. Because this mutual information is intractable for continuous outputs, we derive the Mutual Information Lower Bound (MI-LB) acquisition function, a closed-form approximation for Mixture Density Network ensembles. On benchmarks featuring multimodal systems, MI-LB matches or beats every baseline evaluated and is the only method to do so consistently -- geometric and Fisher-based baselines compete only when the input space already encodes the multimodality, and collapse otherwise.
\end{abstract}

\section{Introduction}
\label{sec:introduction}

Active learning in continuous, multivariate settings requires acquisition functions that can distinguish uncertainty the model can reduce by collecting more data (epistemic) from uncertainty intrinsic to the data-generating process (aleatoric). Existing approaches either operate in restricted settings, such as discrete classification~\cite{BALD-2011} or unimodal Gaussian outputs~\cite{kendall2017uncertainties}, or rely on scalar summaries like predictive variance that conflate the two sources. The problem is worse in multimodal settings: two model hypotheses may assign the same mean and variance to an output while disagreeing on the number and location of modes. Variance alone cannot detect this kind of distributional disagreement.

We address this gap with a framework that keeps the two sources of uncertainty separate by construction. The framework extends the Epistemic Neural Network formalism~\cite{epistemicNNs} by introducing two independent stochastic indices: an epistemic index $Z \sim P_Z$ that parameterizes a family of hypotheses, and an aleatoric index $\epsilon \sim P_\epsilon$ that governs within-hypothesis stochasticity. A learned map $g_\theta(x; z, \epsilon)$ transforms these indices and the input into the output space. This approach provides a common vocabulary for a broad class of uncertainty quantification methods, facilitating the exchange of analyses among model families.

The two-index structure yields an entropy decomposition that separates predictive uncertainty into an epistemic term, the mutual information $I(Y; Z \mid x)$, and an aleatoric term, the expected conditional entropy $\mathbb{E}_Z[H(Y \mid x, Z)]$. Despite being a natural acquisition target, the mutual information requires computing differential entropies that are intractable for general continuous outputs. When the model family is an ensemble of Mixture Density Networks (MDNs), the predictive and conditional distributions are both Gaussian mixtures, and known upper and lower bounds on Gaussian-mixture entropy~\cite{huber2008} can be combined to produce a tractable lower bound on $I(Y; Z \mid x)$. We call this the \textit{Mutual Information Lower Bound} (MI-LB) acquisition function. Across three benchmarks with controllable multimodality, MI-LB matches or beats every baseline evaluated and is the only method to do so consistently. Geometric and Fisher-based baselines compete only when the input geometry already encodes the multimodal structure, and collapse otherwise. To the best of our knowledge, MI-LB is the first acquisition function specifically designed for multimodal continuous settings.

The main contributions of this paper are:
\begin{itemize}
    \item The MI-LB acquisition function for active learning in continuous multivariate settings, with a proof that it lower bounds the true epistemic mutual information $I(Y; Z \mid x)$ and a closed-form expression that is efficient to evaluate under standard MDN assumptions.
    \item The Two-Index framework for disentangling epistemic and aleatoric uncertainty, applicable to a broad class of model families, with a proof that the epistemic term $I(Y; Z \mid x)$ vanishes with sufficient data under standard well-specification and consistency conditions (Appendix~\ref{app:proof_mi_collapse}).    
    \item Experiments on three multimodal benchmarks showing that MI-LB matches or beats every baseline evaluated (Random, Variance, BAIT, Core-Set), and is the only method to do so consistently -- geometric and Fisher-based baselines compete only when the input geometry already encodes the multimodal structure, and collapse when the modal disagreement lives in output space.
\end{itemize}

\subsection{Related Work}
\label{sec:related_work}

Decomposing predictive uncertainty into aleatoric and epistemic components has been a longstanding goal in the machine learning literature~\cite{kendall2017uncertainties, valdenegro2022deeper}. Foundational approaches include Bayesian Neural Networks \cite{Izmailov2021WhatAreBNNPosteriors}, Deep Ensembles \cite{DeepEnsembles2021}, and the Epistemic Neural Network framework \cite{epistemicNNs}, which provides a unified interface that does not require explicit Bayesian inference~\cite{valdenegro2022deeper, epistemicNNs}. Despite these advances, recent benchmarking shows that practical disentanglement remains elusive: modern evidential, variational, and deterministic methods exhibit rank correlations often exceeding 0.94 between their aleatoric and epistemic estimates~\cite{mucsanyi2024benchmarking}. This problem is particularly pronounced in continuous, multimodal regression tasks, where fitting a single Gaussian to diverse target modes inflates predictive variance, motivating the use of Mixture Density Networks (MDNs)~\cite{bishop1994mixture, guilhoto2026multimodalscientificlearningdiffusions}.

In the context of active learning, selecting informative data points to minimize labeling costs has traditionally relied on geometric or gradient-based heuristics~\cite{holzmuller2023framework}. Geometric methods such as Core-Set minimize L2 distances in embedding space~\cite{sener2018active}, while Fisher-based methods like BAIT optimize bounds on the maximum likelihood estimator error using network gradients~\cite{ash2021gone}. Both have important failure modes in multimodal settings: Core-Set ensures coverage in the learned embedding space and can fail when multimodality manifests in the output distribution rather than in the geometry of the inputs~\cite{sener2018active}, and BAIT's last-layer Fisher embedding provides poor candidate rankings when the predictive likelihood is highly multimodal, since the gradient signal at a single Monte-Carlo sample is dominated by within-mode variance. In these regimes, black-box acquisition strategies based on the empirical predictive covariance of ensembles often outperform white-box gradient methods~\cite{kirsch2022sbal, holzmuller2023framework}.

Information-theoretic approaches such as Bayesian Active Learning by Disagreement (BALD) use mutual information to target epistemic uncertainty directly~\cite{BALD-2011, kirsch2019batchbald}. BALD is effective in discrete classification, but extending it to continuous regression requires evaluating differential entropy for distributions like Gaussian mixtures, which lacks a closed-form solution~\cite{huber2008}.

\section{The Two-Index Approach for Disentangling Uncertainties}
\label{sec:disentanglement}


\subsection{The Two-Index Generative Model}
\label{sec:two_index_model}

We build on the Epistemic Neural Network (ENN) framework~\cite{epistemicNNs} and the distribution network formalism of~\cite{guilhoto2026multimodalscientificlearningdiffusions} to construct a unified model that explicitly represents both epistemic and aleatoric uncertainty through two independent stochastic indices.

In the ENN framework, a model is specified by a parameterized function $f_\theta$ and a fixed reference distribution $P_Z$. An \textit{epistemic index} $Z \sim P_Z$ is sampled and passed to the network alongside the input $x$, producing a prediction $f_\theta(x; Z)$. Crucially, $P_Z$ is not updated during training. Instead, training optimizes $\theta$ so that the mapping from $Z$ to predictions becomes increasingly invariant under $P_Z$ as more data is collected. Sensitivity of predictions to $Z$ reflects epistemic uncertainty: the model has not yet identified a unique function consistent with the data.

We extend this framework by introducing the \textit{aleatoric index} $\epsilon \sim P_\epsilon$, a second fixed reference distribution independent of both $Z$ and $\theta$. For a given epistemic index $Z = z$ and input $x$, the model defines a conditional distribution over outputs $Y$ through the generative mapping
\begin{equation}
    Y = g_\theta(x;\, z,\, \epsilon), \qquad Z \sim P_Z, \quad \epsilon \sim P_\epsilon, \quad Z \perp \epsilon.
    \label{eq:generative}
\end{equation}
Here $g_\theta$ is a learned transformation that maps the input $x$ and two independent sources of randomness into the output space. The role of $\epsilon$ is to represent irreducible stochasticity: for any fixed $z$, sampling $\epsilon \sim P_\epsilon$ produces draws from the aleatoric conditional distribution $p_\theta(y \mid x, z)$. The role of $Z$ is to represent reducible uncertainty: variation in predictions across different realizations of $Z$ reflects hypotheses about the data-generating process that have not yet been eliminated by the available data. The network $g_\theta$ learns a transformation from the tractable distributions $P_Z$ and $P_\epsilon$ into a flexible, potentially multimodal distribution over the output $Y$.

\paragraph{Induced distributions.}
The joint generative model in (\ref{eq:generative}) induces a hierarchy of distributions, made precise using the push-forward operator. For a measurable space $(\mathcal{Y}, \mathcal{B}(\mathcal{Y}))$ and a measurable map $h : \mathcal{E} \to \mathcal{Y}$, the \textit{push-forward} of a measure $\mu$ on $\mathcal{E}$ under $h$ is the measure $h_\# \mu$ on $\mathcal{Y}$ defined by
\begin{equation}\label{eq:push-forward}
    (h_\# \mu)(A) := \mu\bigl(\{\epsilon \in \mathcal{E} : h(\epsilon) \in A\}\bigr), \qquad A \in \mathcal{B}(\mathcal{Y}).
\end{equation}
For a fixed epistemic realization $z \in \mathcal{Z}$ and input $x \in \mathcal{X}$, the \textit{aleatoric conditional distribution} is defined as the push-forward of $P_\epsilon$ under the map $g_\theta(x; z, \cdot) : \mathcal{E} \to \mathcal{Y}$:
\begin{equation}
    p_\theta(\cdot \mid x, z) \;:=\; \bigl(g_\theta(x;\, z,\, \cdot)\bigr)_\# P_\epsilon.
    \label{eq:aleatoric_conditional}
\end{equation}
Marginalizing equation (\ref{eq:aleatoric_conditional}) over the epistemic index yields the \textit{predictive distribution}
\begin{equation}
    p_\theta(\cdot \mid x) \;:=\; \int_{\mathcal{Z}} p_\theta(\cdot \mid x, z)\, dP_Z(z).
    \label{eq:predictive}
\end{equation}
The predictive distribution conflates both sources of uncertainty. Disentangling them requires tracking the two indices separately, as formalized in the decompositions of Section~\ref{sec:decompositions}.

\paragraph{Interpretation.}
\label{sec:interpretation}
Each realization $Z = z$ instantiates a \textit{hypothesis} $p_\theta(\cdot \mid x, z)$ for the true conditional $p^*(\cdot \mid x)$; within a fixed hypothesis, $\epsilon$ governs irreducible stochasticity. Disagreement across realizations of $Z$ is the origin of epistemic uncertainty. The predictive distribution has the same mathematical form as Bayesian Model Averaging~\cite{hoeting1999bayesian}, but the framework does not require a strict Bayesian formulation of $P_Z$: as illustrated in Appendix~\ref{app:two_index_examples}, it accommodates ensembles, BNNs, conditional flow matching, and VAEs under a common vocabulary.

\subsection{Quantifying the Two Sources of Uncertainty}
\label{sec:decompositions}

The two-index structure enables exact decompositions of predictive uncertainty into its epistemic and aleatoric components. We provide two such decompositions via variance and entropy.

\paragraph{Variance-Based Decomposition. }\label{sec:var_decomp}
In the scalar output ($Y\in\R$) case~\cite{valdenegro2022deeper}, the law of total variance yields an exact additive decomposition. Writing $Y = g_\theta(x; Z, \epsilon)$ with $Z \perp \epsilon$,
\begin{equation}
    \underbrace{\mathrm{Var}_{Z,\epsilon}(Y|x)}_{\text{total variance}} = \underbrace{\mathbb{E}_Z\bigl[\mathrm{Var}_\epsilon(Y \mid Z, x)\bigr]}_{\text{aleatoric variance}} + \underbrace{\mathrm{Var}_Z\bigl(\mathbb{E}_\epsilon[Y \mid Z, x]\bigr)}_{\text{epistemic variance}}.
    \label{eq:var_decomp}
\end{equation}
The aleatoric term $\mathbb{E}_Z[\mathrm{Var}_\epsilon(Y \mid Z, x)]$ averages the output variance within each hypothesis over epistemic realizations; it measures the average irreducible spread. The epistemic term $\mathrm{Var}_Z(\mathbb{E}_\epsilon[Y \mid Z, x])$ measures the variance of the conditional mean across hypotheses; it captures disagreement between different epistemic realizations of the model.

This decomposition is exact and interpretable, but it summarizes uncertainty through second moments alone. As noted in the introduction to this section, two hypotheses can share the same conditional mean and variance while assigning probability mass to entirely different modes. The variance-based epistemic term would report an incomplete summary in such a case. For multimodal or non-Gaussian distributions, a richer summary is needed, as further discussed in Appendix \ref{app:variance_failure}.

\paragraph{Entropy-Based Decomposition. }\label{sec:entropy_decomp}

We complement (\eqref{eq:var_decomp}) with an entropy-based decomposition that captures distributional disagreement beyond second moments. Define the total predictive uncertainty as the differential entropy of $p_\theta(y \mid x)$:
\begin{equation}
    H_{Z,\epsilon}(Y \mid x) = \mathbb{E}_{Z,\epsilon}\bigl[-\log p_\theta(Y \mid x)\bigr].
\end{equation}
Applying the chain rule of mutual information, this decomposes as
\begin{equation}
    H_{Z,\epsilon}(Y \mid x) = \underbrace{\mathbb{E}_Z\bigl[H_\epsilon(Y \mid Z, x)\bigr]}_{\text{aleatoric uncertainty}} + \underbrace{I(Y;\, Z \mid x)}_{\text{epistemic uncertainty}},
    \label{eq:entropy_decomp}
\end{equation}
where $I(Y; Z \mid x)$ is the \emph{mutual information} between the output $Y$ and the epistemic index $Z$, conditioned on the input $x$. The aleatoric term $\mathbb{E}_Z[H_\epsilon(Y \mid Z, x)]$ is the expected entropy of the conditional distribution $p_\theta(\cdot \mid x, z)$, averaged over the epistemic index. The epistemic term $I(Y; Z \mid x)$ quantifies how much knowing $Z$ reduces uncertainty about $Y$: it is large when different epistemic indices yield meaningfully different conditional distributions, and vanishes when all hypotheses agree. Because mutual information is sensitive to the full shape of each hypothesis distribution, it detects the kind of modal disagreement that variance misses.

\subsection{Asymptotic Guarantees}\label{sec:asymptotics}

The decompositions (\ref{eq:var_decomp}) and (\ref{eq:entropy_decomp}) have the correct asymptotic behavior by construction. As the size of the training dataset grows, $\theta$ is optimized to produce predictions that are consistent with the data under all probable realizations of $Z$. In the limit of infinite data and a well-specified model, the learned $\theta$ renders $g_\theta(x; z, \epsilon)$ approximately independent of $z$ for $P_Z$-almost all $z$, so that:
\begin{equation}
    I(Y;\, Z \mid x,\, \mathcal{D}_n) \xrightarrow{n \to \infty} 0, \qquad \mathbb{E}_Z\bigl[H_\epsilon(Y \mid Z, x,\, \mathcal{D}_n)\bigr] \xrightarrow{n \to \infty} H(p^*(\cdot \mid x)),
\end{equation}
where $p^*(\cdot \mid x)$ is the true conditional distribution of the data-generating process. Epistemic uncertainty vanishes as the model identifies the correct hypothesis, while aleatoric uncertainty converges to the irreducible entropy of the true process. Conversely, at an out-of-distribution input $x_{\mathrm{OOD}}$, training data provides no constraints on $Z$, so $I(Y; Z \mid x_{\mathrm{OOD}})$ can remain large, providing a principled basis for out-of-distribution detection. We note that the convergence of the aleatoric term to $H(p^*(\cdot \mid x))$ requires that the aleatoric model class --- the family of distributions representable by $g_\theta(x; z, \cdot)$ for fixed $z$ --- be sufficiently expressive to capture $p^*$. This is a non-trivial condition that motivates the use of flexible density estimators, such as Mixture Density Networks or implicit generative models, rather than restricting to unimodal Gaussian outputs.

This is stated informally as the following theorem, which is made precise and proved in Appendix~\ref{app:proof_mi_collapse}:
\begin{thm}[Informal: Epistemic Uncertainty Vanishes with Data]\label{thm:mi_collapse_informal}
If a model can represent the true data-generating process $p^*(\cdot \mid x)$ and its parameters converge to the correct values as the dataset grows to infinity, then $I(Y; Z \mid x) \to 0$. In other words, a well-trained model's predictions eventually agree across all epistemic index values, and all remaining uncertainty is aleatoric.
\end{thm}

\subsection{Implications for Data Acquisition}\label{sec:acquisition_implications}

The decomposition from (\ref{eq:entropy_decomp}) and Thm. \ref{thm:mi_collapse_informal} together identify $I(Y; Z \mid x)$ as the natural acquisition target for active learning. It isolates exactly the component of predictive uncertainty that additional data can reduce: epistemic disagreement among hypotheses. The variance-based epistemic term from (\ref{eq:var_decomp}) targets the same goal in principle, but summarizes disagreement through second moments alone; it can miss regions where hypotheses differ in modal structure rather than spread. The mutual information $I(Y; Z \mid x)$, by contrast, is sensitive to the full distribution of each hypothesis, making it the right objective when the conditional distribution $p^*(y \mid x)$ may be multimodal.

The remaining challenge is computational. Evaluating $I(Y; Z \mid x)$ requires computing differential entropies of the predictive and conditional distributions, which have no closed-form expression for general continuous outputs. In the next section, we show that when the model family is an ensemble of Gaussian mixtures, analytical entropy bounds can be used to construct a tractable lower bound on $I(Y; Z \mid x)$ that serves as an effective acquisition function.

\section{The Mutual Information Lower Bound (MI-LB) Acquisition Function}
\label{sec:gee}

\subsection{Entropy Intractability in Continuous Settings}

Re-arranging the entropy-based decomposition from (\ref{eq:entropy_decomp}), the mutual information can be written as
\begin{equation}
    I(Y;\, Z \mid x)
    = \underbrace{H(Y \mid x)}_{\text{(i) marginal entropy}}
    - \underbrace{\mathbb{E}_Z\bigl[H(Y \mid x, Z)\bigr]}_{\text{(ii) expected aleatoric entropy}}.
    \label{eq:mi_split}
\end{equation}
Computing this quantity requires evaluating two differential entropy terms: term (i), the entropy of the predictive distribution $p_\theta(\cdot \mid x)$ from~\eqref{eq:predictive}, and term (ii), the entropy of $Y \mid x, Z=z$ for a fixed $z$, averaged over $P_Z$. For \textit{implicit} distribution networks \cite{guilhoto2026multimodalscientificlearningdiffusions} --- models that produce samples of $Y$ via the map $g_\theta(x; z, \epsilon)$ but do not provide an analytical form for $p_\theta(\cdot \mid x, z)$ --- both terms must be estimated from samples using kernel density estimators, which may scale poorly with the dimension of $\mathcal{Y}$. For \textit{explicit} distribution networks, the density $p_\theta(\cdot \mid x, z)$ is available analytically, enabling structured approximation of both terms. We exploit this structure in the case of Mixture Density Networks, where each $p_\theta(\cdot \mid x, z)$ is a mixture of Gaussians.

\subsection{MDN Structure and the Marginal Mixture}

Assume $Z$ takes values in the discrete index set $\{1, \dots, n_{\mathrm{ens}}\}$ with $P_Z(Z = z) = w_z$, and suppose that for each $z$, the aleatoric conditional $p_\theta(\cdot \mid x, z)$ is a Gaussian mixture with $K$ components:
\begin{equation}
    p_\theta(y \mid x, z)
    = \sum_{i=1}^{K} \alpha_i^{(z)}(x)\,
    \mathcal{N}\!\left(y;\, \mu_i^{(z)}(x),\, C_i^{(z)}(x)\right),
    \label{eq:per_z_gmm}
\end{equation}
where $\alpha_i^{(z)}(x) > 0$, $\sum_i \alpha_i^{(z)}(x) = 1$, $\mu_i^{(z)}(x) \in \mathbb{R}^N$, and $C_i^{(z)}(x) \in \mathbb{R}^{N \times N}$ is a symmetric positive definite covariance matrix. Here $N$ denotes the dimension of the output space $\mathcal{Y}\subseteq \R^N$. Under this formulation, the predictive marginal distribution (\ref{eq:predictive}) is
\begin{equation}
    p_\theta(y \mid x)
    = \sum_{z=1}^{n_{\mathrm{ens}}} w_z\, p_\theta(y \mid x, z)
    = \sum_{z=1}^{n_{\mathrm{ens}}} \sum_{i=1}^{K}
    \underbrace{w_z\, \alpha_i^{(z)}(x)}_{=:\, \beta_{z,i}(x)}\,
    \mathcal{N}\!\left(y;\, \mu_i^{(z)}(x),\, C_i^{(z)}(x)\right),
    \label{eq:marginal_gmm}
\end{equation}
which is a Gaussian mixture with $n_{\mathrm{ens}} \cdot K$ components and weights $\beta_{z,i}(x) = w_z\, \alpha_i^{(z)}(x)$. Both terms in (\ref{eq:mi_split}) therefore require computing the differential entropy of a Gaussian mixture, for which no closed-form expression exists.

\subsection{Entropy Bounds for Gaussian Mixtures}

We make use of two bounds from~\cite{huber2008}. Let
$p(y) = \sum_{i=1}^{K} \pi_i\, \mathcal{N}(y; \mu_i, C_i)$ be a Gaussian mixture with $K$ components, weights $\pi_i$, means $\mu_i \in \mathbb{R}^N$, and covariances $C_i \in \mathbb{R}^{N \times N}$. \cite{huber2008} shows that the differential entropy $H(p)$ satisfies $H_{\mathrm{lower}} \leq H(p) \leq H_{\mathrm{upper}}$, where
\begin{align}
    H_{\mathrm{lower}}
    &= -\sum_{i=1}^{K} \pi_i \log\!\left(
        \sum_{j=1}^{K} \pi_j\,
        \mathcal{N}(\mu_i;\, \mu_j,\, C_i + C_j)
    \right),
    \label{eq:h_lower} \\
    H_{\mathrm{upper}}
    &= \sum_{i=1}^{K} \pi_i \left(
        -\log \pi_i
        + \frac{1}{2}\log\!\left((2\pi e)^N\, |C_i|\right)
    \right).
    \label{eq:h_upper}
\end{align}
Both bounds are computable in $O(K^2)$ operations. The upper bound (\ref{eq:h_upper}) decomposes as the weighted sum of the entropy of individual modes and is tight when the components have negligible overlap. The lower bound (\ref{eq:h_lower}) accounts for inter-component similarity through pairwise Gaussian evaluations at the component means. Under the standard MDN assumption of diagonal covariances $C_i = \mathrm{diag}(\sigma_{i,1}^2, \dots, \sigma_{i,N}^2)$, the Gaussian evaluations in (\ref{eq:h_lower}) and the determinants in (\ref{eq:h_upper}) reduce to sums of univariate quantities, making both bounds particularly efficient to evaluate.

\subsection{The MI-LB Acquisition Function}

We define the \textit{Mutual Information Lower Bound} (MI-LB) acquisition function as
\begin{equation}
    \boxed{
    \textit{MI-LB}(x) :=
    H_{\mathrm{lower}}(Y \mid x)
    -
    \sum_{z=1}^{n_{\mathrm{ens}}} w_z\, H_{\mathrm{upper}}(Y \mid x, Z=z),
    }
    \label{eq:gee}
\end{equation}
where $H_{\mathrm{lower}}(Y \mid x)$ is the lower bound (\ref{eq:h_lower}) applied to the $n_{\mathrm{ens}} \cdot K$-component marginal mixture (\ref{eq:marginal_gmm}) with component weights $\beta_{z,i}(x)$, means $\mu_i^{(z)}(x)$, and covariances $C_i^{(z)}(x)$; and $H_{\mathrm{upper}}(Y \mid x, Z=z)$ is the upper bound (\ref{eq:h_upper}) applied to the $K$-component conditional mixture (\ref{eq:per_z_gmm}) for each $z$. Explicitly,
\begin{align}
    \textit{MI-LB}(x)
    &= -\sum_{z=1}^{n_{\mathrm{ens}}} \sum_{i=1}^{K}
    \beta_{z,i} \log\!\left(
    \sum_{l=1}^{n_{\mathrm{ens}}} \sum_{j=1}^{K}
    \beta_{l,j}\,
    \mathcal{N}\!\left(\mu_i^{(z)};\, \mu_j^{(l)},\,
    C_i^{(z)} + C_j^{(l)}\right)
    \right) \nonumber \\
    &\quad -
    \sum_{z=1}^{n_{\mathrm{ens}}} w_z \sum_{i=1}^{K} \alpha_i^{(z)}
    \left(
    -\log \alpha_i^{(z)}
    + \frac{1}{2}\log\!\left((2\pi e)^N\, |C_i^{(z)}|\right)
    \right),
    \label{eq:gee_explicit}
\end{align}
where we suppress the dependence on $x$ for readability. The fundamental property of MI-LB is that it constitutes a certified lower bound on the true epistemic uncertainty $I(Y; Z \mid x)$, which we state formally as follows, with a proof given in Appendix~\ref{app:proof_gee}.

\begin{thm}[MI-LB is a lower bound on mutual information]
\label{thm:gee_lb}
Under the model assumptions of Section~\ref{sec:disentanglement}, with $Z$ supported on $\{1, \dots, n_{\mathrm{ens}}\}$ and $p_\theta(\cdot \mid x, z)$ a Gaussian mixture of the form (\ref{eq:per_z_gmm}) for each $z$, it holds that
\begin{equation}
    \textit{MI-LB}(x) \leq I(Y;\, Z \mid x)
    \qquad \text{for all } x \in \mathcal{X}.
\end{equation}
\end{thm}

\noindent The practical consequence of Theorem~\ref{thm:gee_lb} is that maximizing MI-LB over a candidate pool is a conservative acquisition strategy: any point selected by MI-LB is guaranteed to have true epistemic uncertainty at least as large as the reported score. This mirrors the logic of the Evidence Lower BOund (ELBO) in variational inference~\cite{kingma2013auto}, where optimizing a lower bound on the log-evidence provides a principled surrogate without risking false certification of a poor solution. Consequently, the MI-LB acquisition rule selects
\begin{equation}
    x^* = \operatorname*{arg\,max}_{x\, \in\, \mathcal{X}_{\mathrm{pool}}}\;
    \textit{MI-LB}(x).
    \label{eq:gee_acquisition}
\end{equation}

\paragraph{Connection to BALD.} The Bayesian Active Learning by Disagreement (BALD) acquisition~\cite{BALD-2011} is defined identically to (\ref{eq:mi_split}) in the discrete classification setting, where $Y$ follows a categorical distribution and entropy is Shannon entropy, admitting exact computation. MI-LB can therefore be understood as a continuous, multivariate alternative to BALD for regression, with the \cite{huber2008} entropy bounds replacing exact entropy to restore tractability in the absence of a finite output alphabet.

\section{Experiments}
We benchmark MI-LB, the acquisition function proposed in Section~\ref{sec:gee}, against four baselines on three multimodal problems: \textbf{Random}, \textbf{Epistemic Variance}, \textbf{BAIT}~\cite{ash2021gone} (last-layer Fisher trace), and \textbf{Core-Set}~\cite{sener2018active} (k-Center-Greedy on the MDN backbone). SBAL~\cite{kirsch2022sbal} and MaxDist~\cite{holzmuller2023framework} batch variants of Variance and MI-LB are deferred to the appendix.
Among the three benchmarks, \S\ref{sec:exp_coupled_double_well} (the coupled double-well) is the discriminating one: input geometry $(\sigma, \kappa)$ does not encode the modal structure, which lives in output space via Kramers escape. \S\ref{sec:exp_multimodal} and \S\ref{sec:exp_ternary_phases} have phase boundaries in the input space and serve to characterise where geometric baselines remain competitive.

\begin{table}[htbp]
\centering
\small
\caption{Per-benchmark settings. All used a pool of $50{,}000$ inputs, $100$ initial labels, and an ensemble of size $8$.}
\label{tab:exp_protocol}
\begin{tabular}{lcccccc}
\toprule
Benchmark & $\dim x$ & $\dim y$ & $K_{\mathrm{MDN}}$ & rounds $\times$ batch & total labels & $\mathrm{NLL}^{*}$ \\
\midrule
Multimodal conditional(\S\ref{sec:exp_multimodal}) & $10$ & $16$ & $5$ & $20 \times 50$ & $1{,}100$ & $22.98$ \\
Coupled double-well (\S\ref{sec:exp_coupled_double_well}) & $7$ & $20$ & $8$ & $20 \times 50$ & $1{,}100$ & --- \\
Ternary phases (\S\ref{sec:exp_ternary_phases}) & $8$ & $1$ & $4$ & $30 \times 15$ & $550$ & $1.33$ \\
\bottomrule
\end{tabular}
\end{table}

\begin{figure}[htbp]
    \centering
    \begin{subfigure}[t]{0.32\textwidth}
        \includegraphics[width=\linewidth]{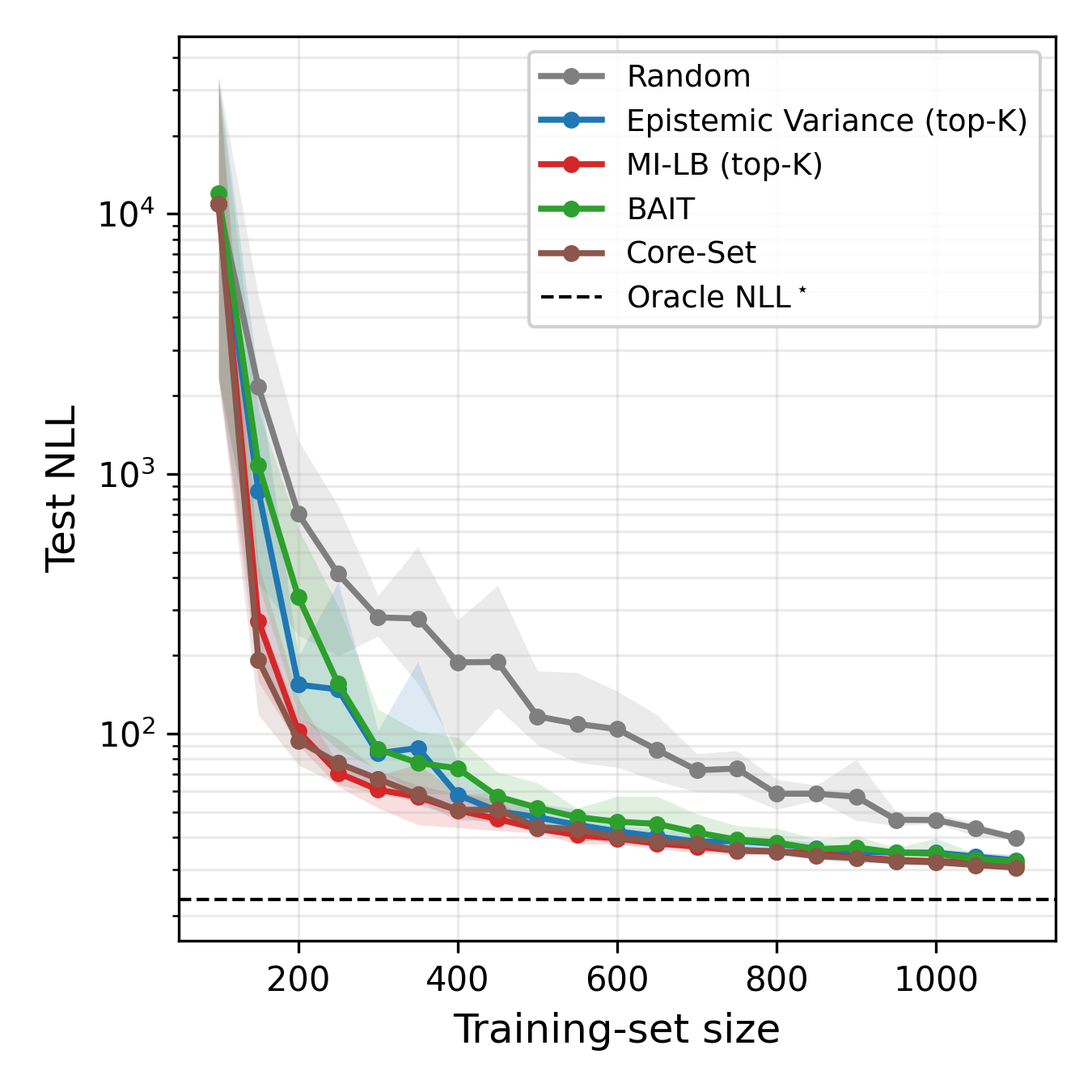}
        \caption{Multimodal conditional (\S\ref{sec:exp_multimodal})}
        \label{fig:multimodal_learning_curves}
    \end{subfigure}\hfill
    \begin{subfigure}[t]{0.32\textwidth}
        \includegraphics[width=\linewidth]{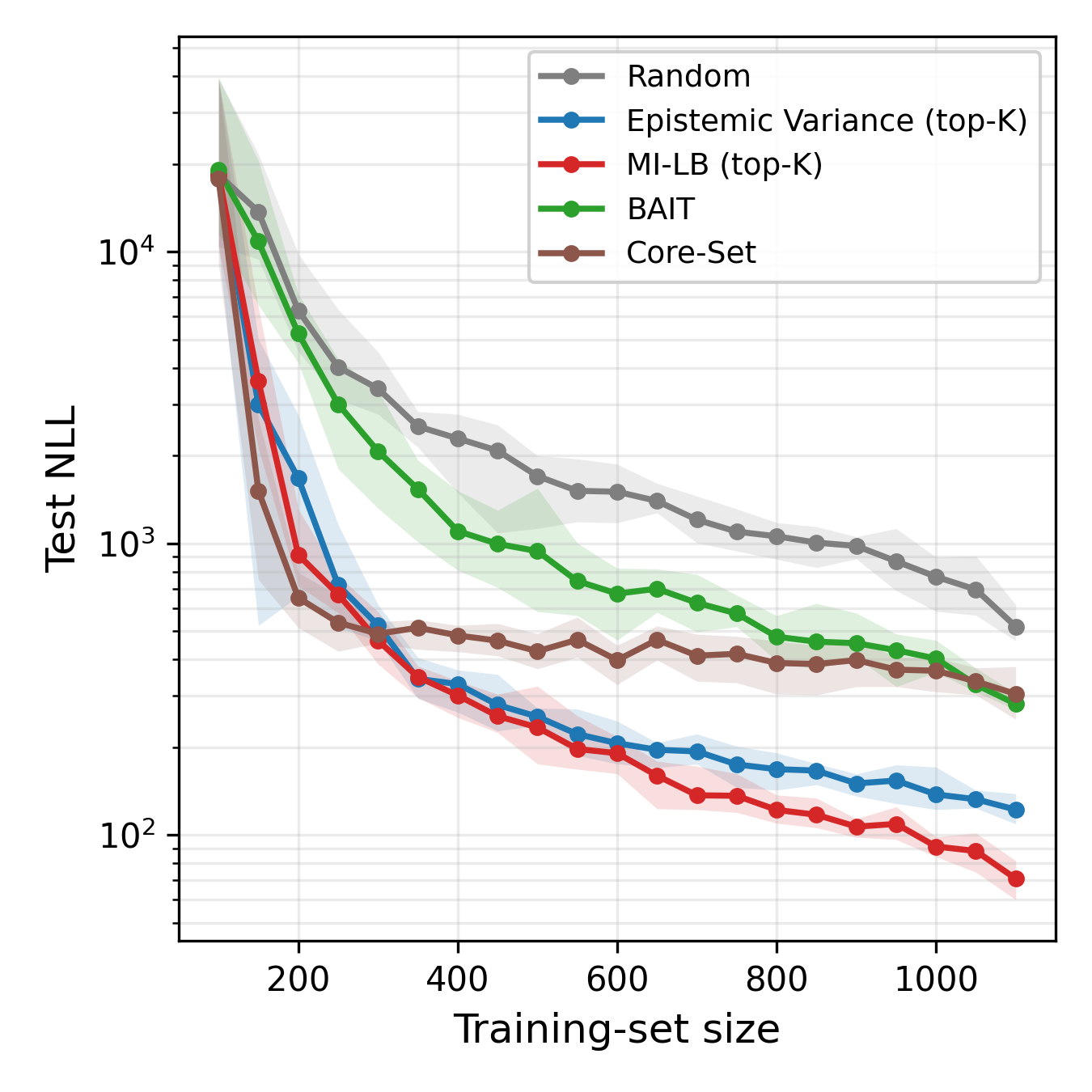}
        \caption{Coupled double-well (\S\ref{sec:exp_coupled_double_well})}
        \label{fig:dw_learning_curves}
    \end{subfigure}\hfill
    \begin{subfigure}[t]{0.32\textwidth}
        \includegraphics[width=\linewidth]{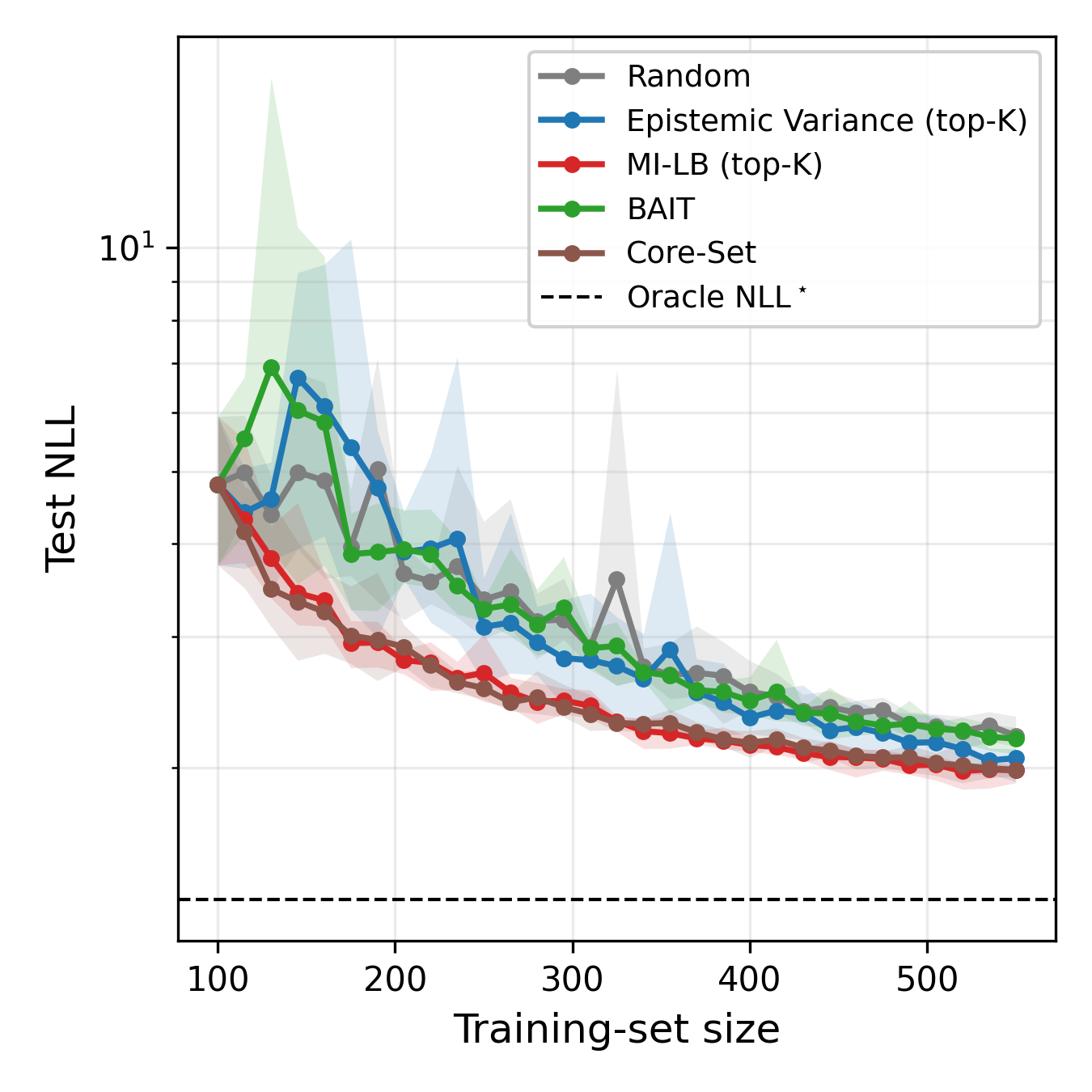}
        \caption{Ternary phases (\S\ref{sec:exp_ternary_phases})}
        \label{fig:ternary_learning_curves}
    \end{subfigure}
    \caption{Test NLL vs.\ training-set size on each benchmark for MI-LB against Random, Epistemic Variance, BAIT, and Core-Set in top-$k$ ($5$ seeds; bands min--max); dashed = oracle $\mathrm{NLL}^{*}$ where available. MI-LB matches or beats every baseline on every benchmark; geometric methods (Core-Set, BAIT) are competitive only when input geometry encodes the multimodality (a, c) and collapse when it lives in output space (b).}
    \label{fig:learning_curves}
\end{figure}

\subsection{Synthetic Multimodal Conditional Problem}
\label{sec:exp_multimodal}

We construct a synthetic conditional distribution $p^*(y \mid x)$ as a mixture
of $K{=}3$ Gaussians:
\begin{equation}
    p^*(y \mid x) = \sum_{k=1}^{K} \pi_k(x)\,
    \mathcal{N}\bigl(y;\, \mu_k(x),\, \Sigma_k(x)\bigr),
    \label{eq:gmm_ground_truth}
\end{equation}
where the mixing weights $\pi_k$, means $\mu_k$, and diagonal covariances $\Sigma_k$ are input-dependent functions (full specification in Appendix~\ref{app:experimental_details}). Inputs $x \in \mathbb{R}^{10}$ lie on a $4$-dimensional manifold embedded via $x = \tanh(Al + b)$, $l \sim \mathcal{N}(0, I_4)$, respecting the manifold hypothesis~\cite{cayton2005algorithms}, and outputs $y \in \mathbb{R}^{16}$. Because $p^*$ is known in closed form, the oracle $\mathrm{NLL}^*$ is computable exactly as a calibration reference. A $K{=}5$ MDN ensemble trained on the full pool reaches NLL $24.72$ vs.\ oracle $22.98$, while a $K{=}1$ MDN attains only $47.76$ -- the conditional is multimodal and a mixture head is necessary (Fig.~\ref{fig:mdn_baseline_samples}). The mixing weights $\pi_k(x)$ are gated by the radial coordinate $r = \|x_{1:L}\|$: for $r \lesssim r_0$ a single component dominates and $p^*(y\mid x)$ is effectively unimodal, while for $r \gtrsim r_0$ multiple components carry comparable mass and $p^*$ becomes genuinely multimodal, producing a sharp unimodal--multimodal transition at $r \approx r_0$ (full form in App.~\ref{app:exp_multimodal}).

\begin{figure}[htbp]
    \centering
    \includegraphics[width=0.8\textwidth]{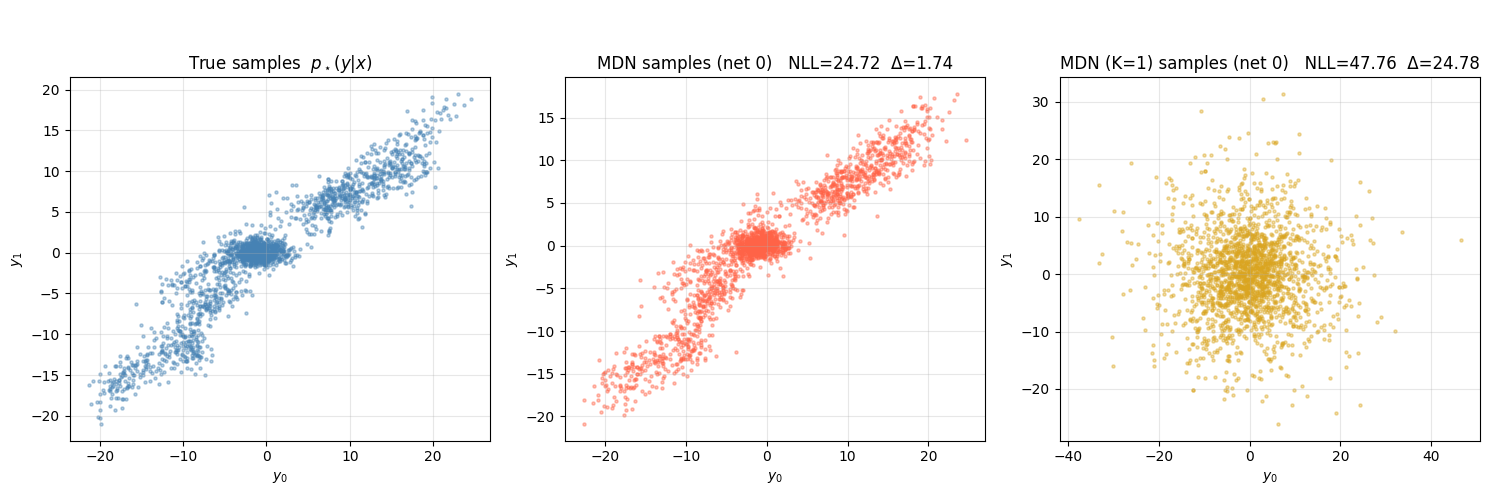}
    \caption{
        Predicted vs.\ true samples in the $(y_0, y_1)$ plane on held-out inputs. \textbf{Left:} draws from the oracle $p^*(y\mid x)$, displaying the multimodal structure of the target conditional. \textbf{Middle:} MDN ensemble samples; the recovered geometry closely matches the oracle, with calibration gap $\Delta = 1.74$. \textbf{Right:} single-Gaussian MDN samples collapses to an isotropic blob that cannot represent disjoint modes, yielding $\Delta = 24.78$.
    }
    \label{fig:mdn_baseline_samples}
\end{figure}

\paragraph{Why this is a useful benchmark.}
The coexistence of unimodal and multimodal regimes makes this problem a natural testbed for the Two-Index framework: $I(Y; Z \mid x)$ should be small in the interior, where all ensemble members agree on a single-component conditional, and large near the transition boundary $\|x_{1:L}\| \approx r_0$, where members must resolve the weights and locations of multiple components. A well-calibrated acquisition function should therefore concentrate queries on this boundary. As a lower bound on $I(Y; Z \mid x)$, MI-LB is designed to detect this distributional disagreement, while variance-based scores summarize ensemble disagreement through second moments alone and underweight regions where hypotheses differ in modal structure rather than spread.
 
\paragraph{Results.}
Figure~\ref{fig:multimodal_learning_curves} compares MI-LB against the four baselines. From the very first acquisition rounds MI-LB pulls clearly ahead of Random, Variance, and BAIT and tracks the strongest baseline (Core-Set) throughout the budget, ending within seed bands at $31.1 \pm 0.3$ vs.\ $30.6 \pm 0.6$ at $n = 1100$. MI-LB matches the best geometric baseline on this benchmark with the tightest seed-to-seed std among all methods, tying Core-Set under top-k and beating it under MaxDist batch selection (Appendix \ref{app:exp_multimodal}, Table \ref{tab:app_multimodal_final_nll}). Core-Set is competitive here for a benchmark-specific reason: the input manifold $x = \tanh(Al + b)$ is $4$-dimensional and the radial coordinate $\|x_{1:L}\|$ already encodes the unimodal--multimodal gate, so k-Center-Greedy in feature space implicitly samples the boundary that MI-LB targets explicitly. Sections~\ref{sec:exp_coupled_double_well}--\ref{sec:exp_ternary_phases} test this dependence and show that MI-LB retains its advantage when geometry no longer encodes multimodality, whereas Core-Set degrades sharply. Spatially, MI-LB concentrates on the unimodal--multimodal boundary, Variance is weaker on the same region, Core-Set covers the manifold uniformly, Random spreads everywhere (App.~\ref{app:exp_multimodal}, Fig.~\ref{fig:app_multimodal_scatter}).

\subsection{Coupled Double-Well System}
\label{sec:exp_coupled_double_well}

We consider a chain of $P = 5$ particles evolving in coupled one-dimensional
double-well potentials under overdamped Langevin dynamics,
\begin{equation}
    dq_i = \Bigl[\underbrace{a\,(q_i - q_i^3)}_{\text{double-well force}}
         + \underbrace{\kappa \sum_{j \in \mathrm{nn}(i)} (q_j - q_i)}_{\text{nearest-neighbour coupling}}\Bigr]\, dt
         + \sigma\, dW_i,
    \qquad i = 1, \dots, P,
    \label{eq:dw_sde}
\end{equation}
where $dW_i$ are independent Wiener increments and we fix $a = 1$. Inputs encode the initial particle configuration together with $(\sigma, \kappa)$; outputs stack particle positions at four snapshot times. Full configuration, integration scheme, and evaluation protocol in Appendix~\ref{app:exp_coupled_double_well}.

\paragraph{Why this is a useful benchmark.}
The conditional $p^*(y \mid x)$ undergoes a sharp change in mode structure controlled by $\sigma^2/a$: at low $\sigma$ each particle stays in its starting well ($p^*$ unimodal per particle); at high $\sigma$, noise triggers \emph{Kramers escape}~\cite{kramers1940brownian}, making the single-particle marginal bimodal. Coupling $\kappa$ raises the effective barrier for collective flips, coupling single-particle marginals into a joint distribution over $\{-1,+1\}^P$ basins. The Kramers exponent $a/(2\sigma^2)$ at barrier height $\Delta V = a/4$ places the crossover at $\sigma \gtrsim \sqrt{a/2} \approx 0.71$. A single-particle sweep (Fig.~\ref{fig:dw_phase_transition}) confirms this scaling, and at $\sigma \gtrsim 1$ noise overwhelms the barrier so that mass spreads beyond both wells into the $|q| > 1$ tails.

\begin{figure}[htbp]
    \centering
    \includegraphics[width=\textwidth]{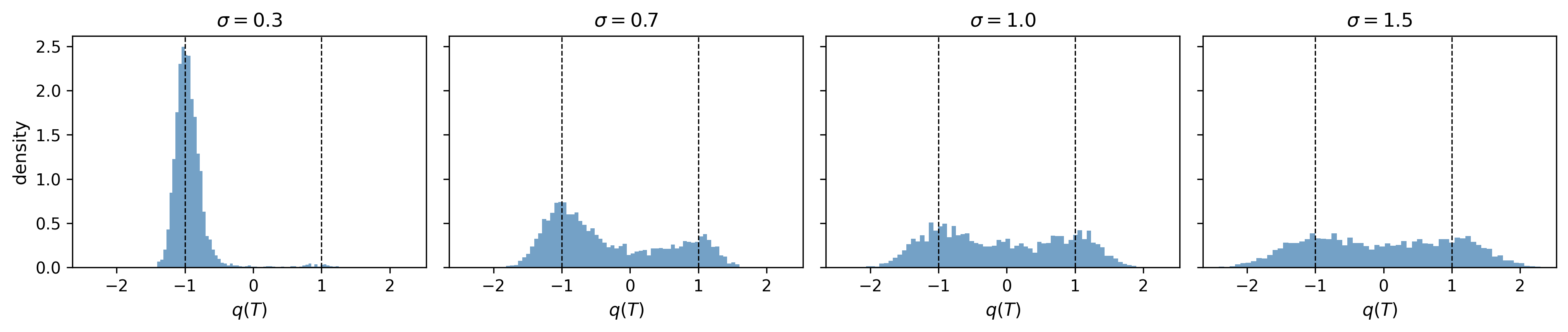}
    \caption{
        Terminal position $q(T)$ histograms for an uncoupled particle ($P=1$, $\kappa=0$, $q(0)=-0.5$) at four noise levels; dashed lines mark $q=\pm 1$. At $\sigma=0.3$ the particle stays trapped near $q=-1$; at $\sigma=0.7 \approx \sqrt{a/2}$ Kramers escape fills both wells; for $\sigma \geq 1$ noise dominates the barrier, spreading mass into the $|q|>1$ tails.
    }
    \label{fig:dw_phase_transition}
\end{figure}

With $\sigma$ and $\kappa$ as input coordinates, the phase boundary lives directly in input space, marking where informative acquisitions should concentrate. A Gaussian head must smear mass across both wells in the bimodal regime, inflating test NLL: a single-Gaussian ($K=1$) head pays a $5.51$-nat penalty against the $K=8$ mixture head used on this benchmark (Appendix~\ref{app:exp_coupled_double_well}, Fig.~\ref{fig:app_dw_mdn_baseline}). This penalty is a property of the benchmark, not of any acquisition strategy: any method built on a single-Gaussian head fails here by construction.

\paragraph{Results.}
Figure~\ref{fig:dw_learning_curves}: MI-LB wins decisively at $\mathbf{71 \pm 8}$ at $n = 1100$ - $1.7\times$ better than Variance ($122 \pm 11$), $4\times$ better than Core-Set ($304 \pm 55$) and BAIT ($281 \pm 15$), and $7\times$ better than Random ($518 \pm 60$). This is the discriminating benchmark: with the bimodal regime's mode structure living in \emph{output} space (Kramers escape between wells), two ensemble members can agree on $\mathbb{E}[Y \mid x]$ while disagreeing on whether mass concentrates near $q = +1$, $q = -1$, or both -- so feature-space coverage (Core-Set) and last-layer Fisher (BAIT) cannot rank candidates informatively, whereas MI-LB's entropy decomposition over the full output mixture targets exactly the disagreement that matters.

\subsection{Synthetic Phase-Competition Benchmark}
\label{sec:exp_ternary_phases}

Active learning is a standard tool in materials discovery, where labels are expensive and composition--process inputs live on low-dimensional manifolds~\cite{lookman2019active,xue2016accelerated,stach2021autonomous}. As in the previous benchmarks, discrete phases compete across the input space, producing boundary-localised multimodality. We build a synthetic CALPHAD-style~\cite{lukas2007calphad} benchmark that reproduces this structure.

Each input is a composition on the $3$-component simplex (with $x_C = 1 - x_A - x_B$) concatenated with $n_{\mathrm{proc}} = 6$ continuous process parameters. Per-phase Gibbs free energies $G_\phi(x_A, x_B, x_C)$ are random quadratic forms in composition (independent of $p$), and the latent phase distribution is $\mathrm{softmax}(-G_\phi / \tau_G)$ at temperature $\tau_G = 0.08$, producing sharp phase boundaries on the simplex slice. The scalar response is sampled from a per-phase Gaussian whose mean depends on both composition and process parameters; full simulator parameters in Appendix~\ref{app:exp_ternary_phases}.

\paragraph{Why this is a useful benchmark.}
The composition $(x_A, x_B, x_C)$ acts as a categorical latent selecting a phase, confining multimodality to a thin boundary region on the simplex; the process subspace contributes only a smooth, phase-conditional shift. This factorisation isolates a clean test of boundary localisation (realised phase masses and boundary fraction in Appendix~\ref{app:exp_ternary_phases}).

\paragraph{Results.}
The tighter budget than Sections~\ref{sec:exp_multimodal} and~\ref{sec:exp_coupled_double_well} reflects the $1$-D scalar output, on which NLL saturates faster than the $16$- and $20$-dimensional outputs of the other benchmarks. Figure~\ref{fig:ternary_learning_curves}: MI-LB reaches $\mathbf{1.99 \pm 0.04}$ at $n = 550$, tied with Core-Set ($1.99 \pm 0.04$); Variance ($2.06$), BAIT ($2.19$), and Random ($2.21$) trail. MaxDist on top of MI-LB narrowly wins overall ($1.95 \pm 0.06$). BAIT collapses to Random-level, the $1$-D scalar output makes its single-MC last-layer Fisher estimate too noisy to rank candidates. Core-Set's parity again reflects benchmark geometry (phase boundaries lie on the composition simplex) and Sections~\ref{sec:exp_coupled_double_well} shows it does not transfer when multimodality is in output space. MI-LB retains the tightest seed-to-seed spread in the data-scarce regime ($n \in [115, 250]$): $\sim\!2.6\times$ tighter than Random, $\sim\!4.5\times$ than Variance (Appendix Table~\ref{tab:app_ternary_seed_ranges}).

\section{Conclusion}

\paragraph{Summary. } We introduce the Two-Index generative framework to formally disentangle epistemic and aleatoric uncertainty across diverse model families. By leveraging this formalism, we derive the MI-LB acquisition function, providing a principled and computationally efficient Mutual Information lower bound for active learning in continuous, multimodal settings. Across three benchmarks MI-LB matches or beats every baseline we evaluate (Random, Variance, BAIT, Core-Set) and is the only method that does so consistently -- the geometric baselines are competitive only on benchmarks whose inputs already encode the multimodality. The discriminating test is the coupled double-well, where Kramers escape places the modal disagreement in output space: Core-Set and BAIT both collapse to within $1.7\times$ of Random, while MI-LB wins by $4\times$ -- the empirical signature that distributional acquisition is doing something feature-space coverage cannot.

\paragraph{Limitations. } MI-LB currently relies on the availability of explicit density components, such as those in Mixture Density Networks. Its performance is also tied to the quality of the entropy bounds, which may loosen in some settings. Furthermore, the framework assumes a well-specified model class capable of recovering the true irreducible stochasticity. Our empirical evaluation is also restricted to synthetic benchmarks with controllable multimodality, since no standard benchmarks exist for active learning in continuous multimodal regression; building such benchmarks is left to future work.

\paragraph{Future Work. } Future research may focus on extending the MI-LB objective to implicit generative models, such as diffusion and flow matching, where densities are not analytically tractable. Another promising direction is to explore the framework’s utility in safety-critical control tasks where distinguishing between lack of data and inherent noise is vital for risk-aware planning.

\begin{ack}
We would like to acknowledge support from the US Department of Energy under the Advanced Scientific Computing Research program (grant DE-SC0024563) and the US National Science Foundation (NSF) Soft AE Research Traineeship (NRT) Program (NSF grant 2152205). We also thank the developers of software that enabled this research, including JAX \cite{jax2018github}, Flax \cite{flax2020github} Matplotlib \cite{Hunter2007matplotlib} and NumPy \cite{harris2020numpy}.
\end{ack}

\section*{Author Contributions}
L.F.G. developed the theory for the two-index framework and MI-LB acquisition. A.K. developed the benchmarks, carried out experiments, and analyzed the data. P.P. provided funding and supervised this study. All authors helped to write and review the manuscript.

\section*{Impact Statement}
Our contributions enable advances in deep learning and uncertainty quantification. This has the potential to impact a wide range of downstream applications. While we do not anticipate specific negative impacts from this work, as with any powerful predictive tool, there is potential for misuse. We encourage the research community to consider the ethical implications and potential dual-use scenarios when applying these technologies in sensitive domains and to avoid its application altogether to weaponry and other military technologies.

\section*{Declaration of LLM Usage}
LLMs were used during the development of this paper, including for editing the writing, developing ideas, and as code assistants. We believe our use to be within the standard uses of this technology, and authors have verified the integrity of all material contained in this work.

\bibliographystyle{unsrt}
\bibliography{references}

\medskip


\clearpage

\appendix

\section{Mathematical Notation}
\label{app:mathematical_notation}

Table \ref{tab:math-notation} summarizes the symbols and notation used in this work.

For operands that involve expectations, such as expectation $\mathbb{E}$, variance $\mathrm{Var}$ and entropy $H$, a sub-index indicates what is the random variable for which the expectation is being taken over. When the context is clear, this sub-index is often omitted.

\begin{table}[ht]
\caption{Summary of the symbols and notation used in this paper.}
\label{tab:math-notation}
\begin{center}
        
\begin{tabular}{ll}\toprule
    \textbf{Symbol} & \textbf{Meaning} \\ \toprule
    $x \in \mathcal{X}$ & Input to the model\\ \midrule
    $Y \in \mathcal{Y}$ & Output random variable\\ \midrule
    $\theta \in \Theta$ & Parameters of a neural network\\ \midrule
    $Z \sim P_Z$ & Epistemic index over model hypotheses\\ \midrule
    $\epsilon \sim P_\epsilon$ & Aleatoric index for irreducible stochasticity\\ \midrule
    $g_\theta(x;\, z,\, \epsilon)$ & Learned generative map to the output space\\ \midrule
    $p_\theta(\cdot \mid x, z)$ & Aleatoric conditional distribution for fixed $z$ from eq. (\ref{eq:aleatoric_conditional})\\ \midrule
    $p_\theta(\cdot \mid x)$ & Predictive distribution, marginal over $P_Z$ from eq. (\ref{eq:predictive})\\ \midrule
    $p^*(\cdot \mid x)$ & True conditional distribution\\ \midrule
    $h_\# \mu$ & Push-forward of measure $\mu$ under map $h$ from eq. (\ref{eq:push-forward})\\ \midrule
    $\mathbb{E}$ & Expectation operator\\ \midrule
    $\mathrm{Var}(\cdot)$ & Variance operator\\ \midrule
    $H(\cdot)$ & Differential entropy\\ \midrule
    $I(Y;\, Z \mid x)$ & Epistemic mutual information\\ \midrule
    $\mathrm{KL}(p \,\|\, q)$ & KL divergence from $q$ to $p$\\ \midrule
    $\mathrm{TV}(p, q)$ & Total variation distance between $p$ and $q$\\ \midrule
    $\mathcal{N}(\mu, \Sigma)$ & Normal distribution with mean $\mu$ and covariance $\Sigma$\\\midrule
    $\mathcal{N}(y;\, \mu, \Sigma)$ & PDF of $\mathcal{N}(\mu, \Sigma)$ evaluated at $y$\\\midrule
    $n_{\mathrm{ens}}$ & Number of ensemble members\\ \midrule
    $K$ & Number of components in a Gaussian mixture\\ \midrule
    $\alpha_i^{(z)}(x)$ & Mixture weight of component $i$ for member $z$\\ \midrule
    $\mu_i^{(z)}(x)$ & Mean of component $i$ for member $z$\\ \midrule
    $C_i^{(z)}(x)$ & Covariance of component $i$ for member $z$\\ \midrule
    $w_z$ & Weight of ensemble member $z$ under $P_Z$\\ \midrule
    $\textit{MI-LB}(x)$ & Mutual Information Lower Bound acquisition function\\ \midrule
    $\mathcal{D}_n$ & Training dataset of size $n$\\ \midrule
    $\mI_d$ & Identity matrix of size $d$\\ \bottomrule
\end{tabular}

\end{center}
\end{table}

\section{Concrete Examples of the Two Index Framework}
\label{app:two_index_examples}

To ground the two-index framework in practice, we show how three families of models already in common use fit naturally into the formalism of equation (\ref{eq:generative}). In each case we identify the epistemic index $Z$, the aleatoric index $\epsilon$, the learned map $g_\theta$, and the resulting conditional and predictive distributions.

\paragraph{Ensemble of Conditional Variational Autoencoders (C-VAEs).}

Consider an ensemble of $n_{\mathrm{ens}}$ independently trained C-VAE  \cite{sohn2015CVAE}, indexed by $k \in \{1, \dots, n_{\mathrm{ens}}\}$. Each model $k$ consists of a decoder $d_{\theta_k} : \mathcal{X} \times \mathcal{L} \to \mathcal{Y}$, where $\mathcal{L}$ is the latent space. At inference time, a latent code is sampled from the prior $\epsilon \sim P_\epsilon := \mathcal{N}(0, I)$ and passed through the decoder to produce a prediction.

In the two-index framework, we set:
\begin{align*}
    Z &\sim P_Z := \mathrm{Uniform}\{1, \dots, n_{\mathrm{ens}}\}, \\
    \epsilon &\sim P_\epsilon := \mathcal{N}(0, I_d), \\
    g_\theta(x;\, z,\, \epsilon) &:= d_{\theta_z}(x,\, \epsilon),
\end{align*}
where $\theta = (\theta_1, \dots, \theta_{n_{\mathrm{ens}}})$ collects all decoder parameters. The epistemic index $Z = k$ selects a member of the ensemble, capturing uncertainty about which model best represents the data-generating process. The aleatoric index $\epsilon$ is the latent noise injected into the decoder of the selected member, capturing the intrinsic stochasticity of the output given a fixed hypothesis. The aleatoric conditional from eq. (\ref{eq:aleatoric_conditional}) becomes
\[
    p_\theta(\cdot \mid x, k) = (d_{\theta_k}(x, \cdot))_\# \mathcal{N}(0, I_d),
\]
which is the generative distribution of the $k$-th C-VAE. The predictive distribution~\eqref{eq:predictive} is the equally-weighted mixture of these per-member distributions:
\[
    p_\theta(\cdot \mid x) = \frac{1}{n_{\mathrm{ens}}} \sum_{k=1}^{n_{\mathrm{ens}}}
    p_\theta(\cdot \mid x, k).
\]
Epistemic uncertainty, measured by $I(Y; Z \mid x)$, is large when the ensemble members disagree on their generative distributions, and vanishes when all decoders produce the same output distribution regardless of which member is selected.

\paragraph{Ensemble of Conditional Flow Matching Models.}

Flow matching models \cite{lipman2022flow} learn a deterministic vector field that transports an initial noise sample $\epsilon \sim P_\epsilon$ to a target distribution over $\mathcal{Y}$, conditioned on input $x$. At inference time, the output is obtained by integrating the learned vector field from time $t=0$ to $t=1$, starting from $\epsilon$. An ensemble of $n_{\mathrm{ens}}$ such models captures epistemic uncertainty through disagreement across members.

The two-index instantiation mirrors the C-VAE case, with a key structural difference: the aleatoric index $\epsilon$ is not a latent code injected at an intermediate layer but rather the \emph{initial noise} supplied to the flow at inference time. Letting $\Phi_{\theta_k}(\cdot; x) : \mathcal{Y} \to \mathcal{Y}$ denote the learned flow map of member $k$ (i.e., the solution operator of the ODE from $t=0$ to $t=1$), we set:
\begin{align*}
    Z &\sim P_Z := \mathrm{Uniform}\{1, \dots, n_{\mathrm{ens}}\}, \\
    \epsilon &\sim P_\epsilon := \mathcal{N}(0, I_d), \\
    g_\theta(x;\, z,\, \epsilon) &:= \Phi_{\theta_z}(\epsilon;\, x).
\end{align*}
The aleatoric conditional is
\[
    p_\theta(\cdot \mid x, k) = (\Phi_{\theta_k}(\cdot\,; x))_\# \mathcal{N}(0, I_d),
\]
which, for an exactly learned flow, equals the target conditional distribution of member $k$.

\paragraph{Bayesian Neural Network for Classification.}

Bayesian Neural Networks (BNNs) are a special case of ENNs~\cite{epistemicNNs}, with the epistemic index $Z$ playing the role of a sample from the weight posterior $p(\theta \mid \mathcal{D})$. We instantiate this within the two-index framework for a classification setting with $n_{\mathrm{cls}}$ classes.

Let $f_w : \mathcal{X} \to \Delta^{n_{\mathrm{cls}}-1}$ be a neural network with weights $w$, mapping inputs to the probability simplex. In a BNN, weights are treated as random variables with posterior $p(w \mid \mathcal{D})$ approximated by a distribution $Q_\theta(w)$ with learnable parameters $\theta$ (e.g., a mean-field Gaussian). The epistemic index $Z$ is a draw from this approximate posterior. The aleatoric index $\epsilon$ serves as the source of randomness for sampling a class label from the predicted categorical distribution. Concretely:
\begin{align*}
    Z &\sim P_Z := Q_\theta(w), \\
    \epsilon &\sim P_\epsilon := \mathrm{Uniform}(0, 1), \\
    g_\theta(x;\, z,\, \epsilon) &:= \min\bigl\{c \in \{1, \dots, n_{\mathrm{cls}}\} : \bigl(F_z(x)\bigr)_c \geq \epsilon\bigr\},
\end{align*}
where $(F_z(x))_c := \sum_{j=1}^c [f_z(x)]_j$ is the $c$-th entry of the cumulative sum of the predicted class probabilities under weights $z$. This is the inverse CDF (quantile) transform: $g_\theta(x; z, \epsilon)$ samples a class label from the categorical distribution $\mathrm{Cat}(f_z(x))$ using $\epsilon$ as the uniform source of randomness. The aleatoric conditional is then
\[
    p_\theta(\cdot \mid x, z) = \mathrm{Cat}(f_z(x)),
\]
and the predictive distribution is equal to the Bayesian Model Average (BMA):
\[
    p_\theta(\cdot \mid x) = \mathbb{E}_{Z \sim Q_\theta}\bigl[\mathrm{Cat}(f_Z(x))\bigr].
\]
The epistemic term $I(Y; Z \mid x)$ measures disagreement among weight posterior samples about the predicted class probabilities, recovering the standard notion of epistemic uncertainty in BNN classification~\cite{epistemicNNs}. The aleatoric term $\mathbb{E}_Z[H(f_Z(x))]$ is the expected entropy of each categorical prediction, averaged over the weight posterior. It measures the label ambiguity that remains under a fixed weight configuration.

\paragraph{Epinet for Regression.}
The epinet \cite{epistemicNNs} augments a base network with a small auxiliary network that takes an epistemic index $z$ as an additional input, so that varying $z$ at a fixed input traces out a family of predictions from a single trained model. Paired with a Gaussian aleatoric head, an epinet captures both sources of uncertainty within one network: the epinet generates the epistemic spread, and the predicted variance captures the aleatoric stochasticity.

Let $\mu_\theta^{\mathrm{base}} : \mathcal{X} \to \mathbb{R}^N$ be the base mean network, $\eta_\theta : \mathcal{X} \times \mathcal{Z}_{\mathrm{epi}} \to \mathbb{R}^N$ be the epinet, and $\sigma_\theta^2 : \mathcal{X} \to \mathbb{R}_{>0}^N$ be the predicted aleatoric variance. The mean prediction at epistemic index $z$ is $\mu_\theta(x, z) := \mu_\theta^{\mathrm{base}}(x) + \eta_\theta(x, z)$. We set:
\begin{align*}
    Z &\sim P_Z := \mathcal{N}(0, I_{d_{\mathrm{epi}}}), \\
    \epsilon &\sim P_\epsilon := \mathcal{N}(0, I_N), \\
    g_\theta(x;\, z,\, \epsilon) &:= \mu_\theta(x,\, z) + \sigma_\theta(x) \odot \epsilon,
\end{align*}
where $\odot$ denotes element-wise multiplication. The aleatoric conditional is
\[
    p_\theta(\cdot \mid x, z) = \mathcal{N}\!\left(\mu_\theta(x, z),\, \mathrm{diag}\bigl(\sigma_\theta^2(x)\bigr)\right),
\]
a Gaussian whose mean depends on $z$ but whose covariance does not. The predictive distribution is the continuous mixture
\[
    p_\theta(\cdot \mid x) = \mathbb{E}_{Z \sim \mathcal{N}(0, I_{d_{\mathrm{epi}}})}\!\left[\mathcal{N}\!\left(\mu_\theta(x, Z),\, \mathrm{diag}\bigl(\sigma_\theta^2(x)\bigr)\right)\right],
\]
which is a Gaussian-mean mixture: the mean varies with $Z$ while the covariance is fixed at $\mathrm{diag}(\sigma_\theta^2(x))$. The epistemic term $I(Y; Z \mid x)$ measures how much the mean prediction $\mu_\theta(x, Z)$ varies with the epistemic index. The aleatoric term has the closed form $\mathbb{E}_Z[H_\epsilon(Y \mid x, Z)] = \frac{1}{2} \log\!\left((2 \pi e)^N \prod_{i=1}^N \sigma_{\theta, i}^2(x)\right)$, since the conditional covariance does not depend on $Z$. The two sources of uncertainty are therefore parameterized by disjoint sub-networks: the epinet $\eta_\theta$ governs the epistemic spread, and the variance head $\sigma_\theta$ governs the aleatoric noise. The same construction can be extended to the operator learning setting, where inputs and outputs are functions rather than vectors~\cite{guilhoto2024neon}.

\section{Limitations of Variance for Multimodal Distributions}
\label{app:variance_failure}

The Epistemic Variance acquisition function (Section~\ref{sec:var_decomp}), also known as Uncertainty Sampling, is the standard choice for active learning with real-valued predictions. There are settings, however, where it assigns identical scores to distributions with very different uncertainty profiles. We give a concrete example on the unit sphere that illustrates this failure mode.

For a distribution $p$ over $\R^m$ with mean $\bar{y} = \mathbb{E}_p[Y]$, we define the \textit{multivariate variance} as the trace of the covariance matrix:
\begin{equation}
    \mathrm{Var}(p) := \mathrm{tr}\bigl(\mathrm{Cov}_p(Y)\bigr) = \mathbb{E}_p\bigl[\|Y - \bar{y}\|_2^2\bigr].
    \label{eq:multivariate_variance}
\end{equation}
This is the expected squared $\ell_2$ distance from a sample to the mean, and it is the quantity that variance-based acquisition functions rank candidates by.

Consider the unit sphere $S^{m-1} := \{y \in \R^m : \|y\|_2 = 1\}$ and two inputs $x, x' \in \R^n$. For $x$, suppose $p_\theta(y \mid x)$ is the \textit{uniform distribution} on $S^{m-1}$. For $x'$, let $C_\delta(e_1)$ and $C_\delta(-e_1)$ denote the spherical caps of geodesic radius $\delta > 0$ centered at $e_1$ and $-e_1$ respectively, and define the \textit{polar distribution} $p_\theta(y \mid x')$ as the distribution that selects one of the two caps with equal probability and then draws uniformly within it. Both distributions are continuous with respect to the surface measure on $S^{m-1}$, both have mean $\bar{y} = 0$ by symmetry, and every sample lies on $S^{m-1}$. Since $\|Y\|_2^2 = 1$ almost surely and $\bar{y} = 0$ for both distributions, equation (\ref{eq:multivariate_variance}) gives $\mathrm{Var}(p) = 1$ in both cases, exactly and for all $m$ (see Figure~\ref{fig:variance-failure} for an illustration). Any acquisition function that ranks inputs by variance assigns the same score to $x$ and $x'$.

\begin{figure}[htbp]
    \centering
    \includegraphics[width=0.95\textwidth]{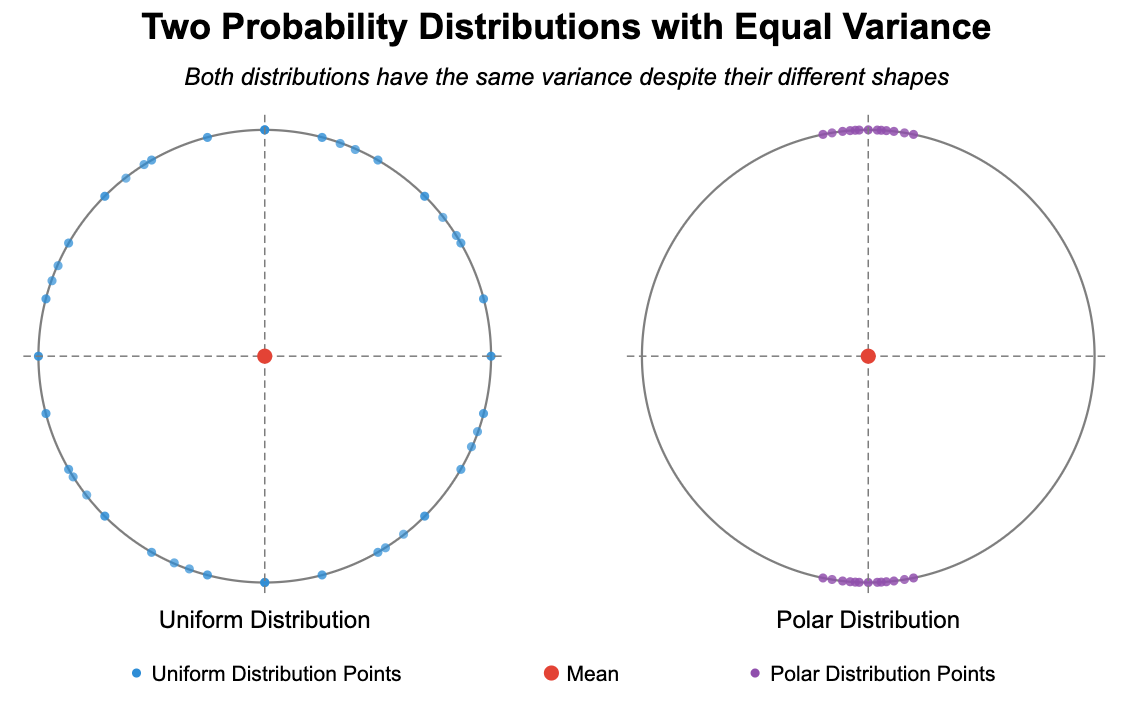}
    \caption{Two distributions on the unit circle with identical variance but different entropy. \textbf{Left:} the uniform distribution spreads samples evenly across the circle. \textbf{Right:} the polar distribution concentrates samples near two antipodal points. Both distributions have mean zero and multivariate variance equal to $1$, yet the uniform distribution has strictly higher entropy.}
    \label{fig:variance-failure}
\end{figure}

The two distributions are not equally uncertain, however. The polar distribution confines the output to two small patches of the sphere, while the uniform distribution spreads mass over the entire surface. The entropy-based decomposition of Section~\ref{sec:entropy_decomp} captures this gap. Writing $A_m = \mathrm{Area}(S^{m-1})$ for the total surface area and $A_\delta = \mathrm{Area}(C_\delta)$ for the area of each cap, the differential entropies with respect to the surface measure are
\begin{equation}
    h\bigl(p_\theta(\cdot \mid x)\bigr) = \log A_m, \qquad h\bigl(p_\theta(\cdot \mid x')\bigr) = \log 2 + \log A_\delta.
    \label{eq:cap_entropies}
\end{equation}
The entropy gap is $\log(A_m / (2 A_\delta))$, which grows without bound as $m \to \infty$ for any fixed cap radius $\delta$. This follows from concentration of measure on the sphere: the fraction $A_\delta / A_m$ of the sphere covered by a cap of fixed geodesic radius $\delta < \pi/2$ vanishes exponentially in $m$, because the mass of the integrand $\sin^{m-2}(\phi)$ in the surface area element concentrates near the equator $\phi = \pi/2$. Entropy-based acquisition functions detect this growing gap, while variance-based scores cannot.

\section{Proof of Asymptotic Guarantees}
\label{app:proof_mi_collapse}

In what follows, we formalize the insights from Section \ref{sec:asymptotics} and theorem \ref{thm:mi_collapse_informal} in order to provide guarantees on the asymptotic behavior under infinite data. First, we prove that epistemic uncertainty measured as the mutual information $I(Y;Z|x)$ converges to zero. Then, we show that the predicted aleatoric uncertainty $\mathbb{E}_Z\bigl[H_\epsilon(Y \mid Z, x,\, \mathcal{D}_n)\bigr]$ converges to the true quantity $H(p^*(\cdot \mid x))$. As a corollary of these two results, we also observe that the total predictive entropy also converges to the true entropy of the data-generating process.

\subsection{Mutual Information Collapse}

We begin by stating and proving a lemma that will be helpful in proving the main theorem.
 
\begin{lem}[KL divergence--total variation bound under bounded likelihood ratio]\label{lem:kl_tv}
Let $p$ and $q$ be probability measures on a measurable space $(\mathcal{Y}, \mathcal{B}(\mathcal{Y}))$ with $p \ll q$. Suppose there exists a constant $C \geq 1$ such that $dp/dq \leq C$ holds $q$-almost everywhere. Then
\[
    \mathrm{KL}(p \,\|\, q) \leq 2M_C \cdot \mathrm{TV}(p, q),
\]
where
\begin{equation}
    M_C := \max\!\left(1,\; \frac{C \log C - C + 1}{C - 1}\right) \label{eq:M_C_def}
\end{equation}
for $C > 1$, and $M_1 := 1$.
\end{lem}
 
\begin{proof}
Let $h := dp/dq$, so that $0 \leq h \leq C$ holds $q$-a.e. Since $\int h\,dq = 1$, the KL divergence can be written as
\begin{equation}
    \mathrm{KL}(p \,\|\, q) = \int h \log h\,dq = \int \phi(h)\,dq, \label{eq:kl_phi}
\end{equation}
where $\phi(t) := t\log t - t + 1$. The function $\phi$ is non-negative and convex on $[0, \infty)$ with $\phi(1) = 0$ and $\phi'(1) = 0$.
 
We claim that $\phi(t) \leq M_C\,|t - 1|$ for all $t \in [0, C]$. On the interval $[0, 1]$: since $t\log t \leq 0$ for $t \in [0,1]$, we have $\phi(t) = t\log t + (1 - t) \leq 1 - t = |t - 1|$. On the interval $[1, C]$: convexity of $\phi$ together with $\phi(1) = 0$ gives $\phi(t) \leq \phi(C)\cdot\frac{t - 1}{C - 1}$, where $\phi(C) = C\log C - C + 1$. Combining both cases yields $\phi(t) \leq M_C\,|t - 1|$ on $[0, C]$.
 
Integrating against $q$ and applying the bound gives
\begin{align}
    \mathrm{KL}(p \,\|\, q)
    = \int \phi(h)\,dq
    &\leq M_C \int |h - 1|\,dq \notag \\
    &= 2M_C \cdot \mathrm{TV}(p, q), \label{eq:kl_tv_final}
\end{align}
where the last equality uses $\mathrm{TV}(p, q) = \frac{1}{2}\int |dp - dq| = \frac{1}{2}\int |h - 1|\,dq$.
\end{proof}

Equipped with this lemma, we now state the first part of the theorem precisely.
 
\begin{thm}[Mutual Information Collapse]\label{thm:mi_collapse}
Let $(\mathcal{Y}, \mathcal{B}(\mathcal{Y}))$ be a measurable space, let $(\mathcal{Z}, P_Z)$ and $(\mathcal{E}, P_\epsilon)$ be probability spaces, let $\Theta \subseteq \mathbb{R}^d$, and let $g_\theta : \mathcal{X} \times \mathcal{Z} \times \mathcal{E} \to \mathcal{Y}$ be measurable for each $\theta \in \Theta$. Fix $x \in \mathcal{X}$, define the aleatoric conditional $p_\theta(\cdot \mid x, z) := (g_\theta(x; z, \cdot))_\# P_\epsilon$ and the predictive marginal $p_\theta(\cdot \mid x) := \int p_\theta(\cdot \mid x, z)\,dP_Z(z)$. Suppose the following conditions hold.
\begin{enumerate}
    \item[A1] \textbf{Well-specification.} There exists $\theta^* \in \Theta$ such that $p_{\theta^*}(\cdot \mid x, z) = p^*(\cdot \mid x)$ for $P_Z$-almost every $z$.
    \item[A2] \textbf{Consistency.} $\hat{\theta}_n \to \theta^*$ in probability as $n \to \infty$.
    \item[A3] \textbf{Uniform continuity in total variation.} The map $\theta \mapsto p_\theta(\cdot \mid x, z)$ is continuous at $\theta^*$ in total variation, uniformly in $z$: for every $\varepsilon > 0$ there exists $\delta > 0$ such that $\|\theta - \theta^*\| < \delta$ implies $\sup_{z \in \mathcal{Z}} \mathrm{TV}(p_\theta(\cdot \mid x, z), p_{\theta^*}(\cdot \mid x, z)) < \varepsilon$.
    \item[A4] \textbf{Bounded likelihood ratio.} There exists $C \geq 1$ such that for all $\theta$ in a neighborhood of $\theta^*$ and $P_Z$-almost every $z$, the aleatoric conditional is absolutely continuous with respect to the predictive marginal and
    \[
        \frac{dp_\theta(\cdot \mid x, z)}{dp_\theta(\cdot \mid x)} \leq C \qquad p_\theta(\cdot \mid x)\text{-a.e.}
    \]
\end{enumerate}
Then
\[
    I_{\hat{\theta}_n}(Y;\, Z \mid x) := \mathbb{E}_Z\!\left[\mathrm{KL}\!\left(p_{\hat{\theta}_n}(\cdot \mid x, Z) \,\Big\|\, p_{\hat{\theta}_n}(\cdot \mid x)\right)\right] \xrightarrow{p} 0 \quad \text{as } n \to \infty.
\]
\end{thm}
 
Before proving this theorem, we describe what assumptions A1--A4 intuitively mean. Assumption A1 relates to \textit{universality} of the chosen model class: the neural network architecture should be expressive enough to contain a parameter $\theta^*$ that captures the desired map, and is a standard assumption in the machine learning literature, having been proved for many choices of architectures via universal approximation theorems \cite{cybenko1989approximation, lu2020universal}. Assumption A2 relates to well-posedness of training: under infinite data, it should be possible to recover the true optimum of the problem. Assumption A3 is a regularity condition requiring that the aleatoric conditionals respond smoothly to parameter perturbations, uniformly across the epistemic index. Assumption A4 requires a uniform upper bound on the likelihood ratio $dp_\theta(\cdot \mid x, z)/dp_\theta(\cdot \mid x)$; for finite ensembles with prior weights $w_z > 0$, this is automatically satisfied with $C = 1/\min_z w_z$, since the marginal $p_\theta(\cdot \mid x) = \sum_z w_z\,p_\theta(\cdot \mid x, z) \geq w_z\,p_\theta(\cdot \mid x, z)$ dominates each conditional. For continuous epistemic indices, A4 is a mild regularity condition excluding pathological concentration of the conditional relative to the marginal.
\begin{proof}
Fix $x \in \mathcal{X}$ and write $I_\theta := I_\theta(Y; Z \mid x)$, $p_\theta(z) := p_\theta(\cdot \mid x, z)$, and $p_\theta := p_\theta(\cdot \mid x)$ for brevity. The strategy is to show that $\theta \mapsto I_\theta$ is continuous at $\theta^*$ with $I_{\theta^*} = 0$, after which the conclusion follows from A2 by the continuous mapping theorem.
 
By A1, $p_{\theta^*}(z) = p^*(\cdot \mid x)$ for $P_Z$-a.e.\ $z$, and marginalizing gives
\begin{align}
    p_{\theta^*}
    = \int p_{\theta^*}(z)\,dP_Z(z)
    &= \int p^*(\cdot \mid x)\,dP_Z(z) \notag \\
    &= p^*(\cdot \mid x), \label{eq:marginal_at_theta_star}
\end{align}
so that $p_{\theta^*}(z) = p_{\theta^*}$ for $P_Z$-a.e.\ $z$. Consequently $\mathrm{KL}(p_{\theta^*}(z) \,\|\, p_{\theta^*}) = 0$ for $P_Z$-a.e.\ $z$, and integrating against $P_Z$ yields $I_{\theta^*} = 0$.
 
For continuity at $\theta^*$, fix $\varepsilon > 0$. Since $p_\theta(\cdot \mid x)$ is a mixture over $P_Z$ with $p_\theta(\cdot \mid x, z) \ll p_\theta(\cdot \mid x)$, Lemma~\ref{lem:kl_tv} together with A4 gives, for every $\theta$ in the neighborhood specified by A4 and $P_Z$-a.e.\ $z$,
\begin{equation}
    \mathrm{KL}(p_\theta(z) \,\|\, p_\theta) \leq 2M_C \cdot \mathrm{TV}(p_\theta(z), p_\theta), \label{eq:kl_to_tv}
\end{equation}
where $M_C$ is the constant from~\eqref{eq:M_C_def}. It therefore suffices to bound $\mathbb{E}_Z[\mathrm{TV}(p_\theta(Z), p_\theta)]$. The triangle inequality for total variation, applied with $p_{\theta^*}(z)$ and $p_{\theta^*}$ as intermediate points, gives
\begin{align}
    \mathrm{TV}(p_\theta(z), p_\theta)
    \leq{}& \mathrm{TV}(p_\theta(z), p_{\theta^*}(z)) \notag \\
    &+ \mathrm{TV}(p_{\theta^*}(z), p_{\theta^*}) \notag \\
    &+ \mathrm{TV}(p_{\theta^*}, p_\theta), \label{eq:tv_triangle}
\end{align}
in which the middle term vanishes for $P_Z$-a.e.\ $z$ by \eqref{eq:marginal_at_theta_star}. By A3 there exists $\delta > 0$ such that $\|\theta - \theta^*\| < \delta$ implies
\begin{equation}
    \sup_{z \in \mathcal{Z}} \mathrm{TV}(p_\theta(z), p_{\theta^*}(z)) < \frac{\varepsilon}{4M_C}, \label{eq:a3_first_term}
\end{equation}
controlling the first term in \eqref{eq:tv_triangle} uniformly in $z$. For the third term, joint convexity of total variation together with the representation $p_\theta = \int p_\theta(z)\,dP_Z(z)$ yields
\begin{align}
    \mathrm{TV}(p_{\theta^*}, p_\theta)
    &\leq \int \mathrm{TV}(p_{\theta^*}(z), p_\theta(z))\,dP_Z(z) \notag \\
    &< \frac{\varepsilon}{4M_C}, \label{eq:third_term_bound}
\end{align}
where the strict bound again follows from \eqref{eq:a3_first_term} at the same $\delta$. Integrating \eqref{eq:tv_triangle} against $P_Z$ and substituting \eqref{eq:a3_first_term} and \eqref{eq:third_term_bound} gives
\begin{align}
    \mathbb{E}_Z[\mathrm{TV}(p_\theta(Z), p_\theta)]
    &< \frac{\varepsilon}{4M_C} + 0 + \frac{\varepsilon}{4M_C} \notag \\
    &= \frac{\varepsilon}{2M_C},
\end{align}
which combined with \eqref{eq:kl_to_tv} yields
\begin{align}
    I_\theta
    \leq 2M_C \cdot \mathbb{E}_Z[\mathrm{TV}(p_\theta(Z), p_\theta)]
    < 2M_C \cdot \frac{\varepsilon}{2M_C}
    = \varepsilon.
\end{align}
Hence $\theta \mapsto I_\theta$ is continuous at $\theta^*$.
 
Since $\hat{\theta}_n \xrightarrow{p} \theta^*$ by A2 and $I_\theta$ is continuous at $\theta^*$ with $I_{\theta^*} = 0$, the continuous mapping theorem gives $I_{\hat{\theta}_n} \xrightarrow{p} I_{\theta^*} = 0$.
\end{proof}

\subsection{Aleatoric Entropy Convergence}
\label{app:aleatoric_consistency}

Theorem~\ref{thm:mi_collapse} establishes that the epistemic mutual information collapses to zero. We now show that the aleatoric uncertainty estimator converges to the entropy of the true conditional. The argument follows the same three-step template as Theorem~\ref{thm:mi_collapse}: point evaluation at $\theta^*$, continuity at $\theta^*$, and the continuous mapping theorem. The new ingredient is a regularity condition that does not follow from A3 alone.

Assumption A3 controls $\theta \mapsto p_\theta(\cdot \mid x, z)$ in total variation, but differential entropy is not continuous in TV: distributions can be arbitrarily TV-close with entropies that differ by any prescribed amount. We therefore add the following assumption.
\begin{enumerate}
    \item[A5] \textbf{Uniform entropy continuity.} The map $\theta \mapsto H(p_\theta(\cdot \mid x, z))$ is continuous at $\theta^*$, uniformly in $z$: for every $\varepsilon > 0$ there exists $\delta > 0$ such that $\|\theta - \theta^*\| < \delta$ implies $\sup_{z \in \mathcal{Z}} \left|H(p_\theta(\cdot \mid x, z)) - H(p_{\theta^*}(\cdot \mid x, z))\right| < \varepsilon$.
\end{enumerate}

For Gaussian-mixture conditionals of the type used in Section~\ref{sec:exp_multimodal}, A5 follows from continuous parametrisation of the mixture weights, means, and covariances together with a uniform lower bound on component covariance eigenvalues in a neighborhood of $\theta^*$. For finite ensembles, the supremum over $z$ reduces to a maximum over finitely many continuous functions and is automatic.

\begin{thm}[Aleatoric Entropy Consistency]\label{thm:aleatoric_consistency}
Adopt the setup of Theorem~\ref{thm:mi_collapse} and suppose A1, A2, and A5 hold. Then
\[
    A_{\hat{\theta}_n}(x) := \mathbb{E}_Z\!\left[H_\epsilon\!\left(Y \mid Z, x;\, \hat{\theta}_n\right)\right] \xrightarrow{p} H(p^*(\cdot \mid x)) \quad \text{as } n \to \infty.
\]
\end{thm}

\begin{proof}
Fix $x \in \mathcal{X}$ and write $A_\theta := \mathbb{E}_Z[H(p_\theta(\cdot \mid x, Z))]$ and $p_\theta(z) := p_\theta(\cdot \mid x, z)$ for brevity. The strategy mirrors Theorem~\ref{thm:mi_collapse}: show that $\theta \mapsto A_\theta$ is continuous at $\theta^*$ with $A_{\theta^*} = H(p^*(\cdot \mid x))$, then apply the continuous mapping theorem under A2.

By A1, $p_{\theta^*}(z) = p^*(\cdot \mid x)$ for $P_Z$-a.e.\ $z$, so $H(p_{\theta^*}(z)) = H(p^*(\cdot \mid x))$ for $P_Z$-a.e.\ $z$. Integrating against $P_Z$ gives
\begin{equation}
    A_{\theta^*} = H(p^*(\cdot \mid x)). \label{eq:aleatoric_at_theta_star}
\end{equation}

For continuity at $\theta^*$, fix $\varepsilon > 0$. By A5 there exists $\delta > 0$ such that $\|\theta - \theta^*\| < \delta$ implies
\begin{equation}
    \sup_{z \in \mathcal{Z}} \left|H(p_\theta(z)) - H(p_{\theta^*}(z))\right| < \varepsilon. \label{eq:a5_uniform}
\end{equation}
Applying the triangle inequality for integrals followed by~\eqref{eq:a5_uniform} yields
\begin{align}
    |A_\theta - A_{\theta^*}|
    &\leq \int \left|H(p_\theta(z)) - H(p_{\theta^*}(z))\right|\,dP_Z(z) \notag \\
    &\leq \sup_{z \in \mathcal{Z}} \left|H(p_\theta(z)) - H(p_{\theta^*}(z))\right| \notag \\
    &< \varepsilon,
\end{align}
so $\theta \mapsto A_\theta$ is continuous at $\theta^*$.

Since $\hat{\theta}_n \xrightarrow{p} \theta^*$ by A2 and $\theta \mapsto A_\theta$ is continuous at $\theta^*$, the continuous mapping theorem combined with~\eqref{eq:aleatoric_at_theta_star} gives $A_{\hat{\theta}_n} \xrightarrow{p} H(p^*(\cdot \mid x))$.
\end{proof}

Theorems~\ref{thm:mi_collapse} and~\ref{thm:aleatoric_consistency} together imply that the total predictive entropy converges to the entropy of the true conditional, which gives a complete consistency picture for the entropy decomposition~\eqref{eq:entropy_decomp}.

\begin{cor}[Total Predictive Entropy Consistency]\label{cor:total_entropy}
Under A1--A5,
\[
    H_{Z,\epsilon}(Y \mid x;\, \hat{\theta}_n) \xrightarrow{p} H(p^*(\cdot \mid x)) \quad \text{as } n \to \infty.
\]
\end{cor}

\begin{proof}
By the entropy decomposition~\eqref{eq:entropy_decomp},
\[
    H_{Z,\epsilon}(Y \mid x;\, \hat{\theta}_n) = A_{\hat{\theta}_n}(x) + I_{\hat{\theta}_n}(Y;\, Z \mid x).
\]
Theorem~\ref{thm:aleatoric_consistency} gives $A_{\hat{\theta}_n}(x) \xrightarrow{p} H(p^*(\cdot \mid x))$ and Theorem~\ref{thm:mi_collapse} gives $I_{\hat{\theta}_n}(Y;\, Z \mid x) \xrightarrow{p} 0$. The continuous mapping theorem applied to the sum yields the result.
\end{proof}

\section{Proof of Theorem~\ref{thm:gee_lb}}
\label{app:proof_gee}

\begin{proof}
We wish to show that $\textit{MI-LB}(x) \leq I(Y; Z \mid x)$ for all $x \in \mathcal{X}$.
Recalling the decomposition from (\ref{eq:mi_split}),
\begin{equation}
    I(Y;\, Z \mid x) = H(Y \mid x) - \mathbb{E}_Z\bigl[H(Y \mid x, Z)\bigr].
    \label{eq:mi_app}
\end{equation}
We bound the two terms on the right-hand side in opposite directions using the results
of~\cite{huber2008}.

\medskip
\noindent\textbf{Step 1: Lower bounding $H(Y \mid x)$.}

The marginal predictive distribution $p_\theta(\cdot \mid x)$ is the Gaussian mixture
with $n_{\mathrm{ens}} \cdot K$ components given in~\eqref{eq:marginal_gmm}. By the
lower bound of~\cite{huber2008},
\begin{equation}
    H(Y \mid x) \geq H_{\mathrm{lower}}(Y \mid x),
    \label{eq:lb_marginal}
\end{equation}
where $H_{\mathrm{lower}}(Y \mid x)$ is~\eqref{eq:h_lower} applied to the marginal
mixture with weights $\beta_{z,i}$, means $\mu_i^{(z)}$, and covariances $C_i^{(z)}$.

\medskip
\noindent\textbf{Step 2: Upper bounding $\mathbb{E}_Z[H(Y \mid x, Z)]$.}

For each fixed $z \in \{1, \dots, n_{\mathrm{ens}}\}$, the aleatoric conditional
$p_\theta(\cdot \mid x, z)$ is a Gaussian mixture with $K$ components as
in~\eqref{eq:per_z_gmm}. By the upper bound of \cite{huber2008},
\begin{equation}
    H(Y \mid x, Z=z) \leq H_{\mathrm{upper}}(Y \mid x, Z=z)
    \qquad \text{for each } z.
\end{equation}
Multiplying by $w_z \geq 0$ and summing over $z$,
\begin{equation}
    \mathbb{E}_Z\bigl[H(Y \mid x, Z)\bigr]
    = \sum_{z=1}^{n_{\mathrm{ens}}} w_z\, H(Y \mid x, Z=z)
    \leq \sum_{z=1}^{n_{\mathrm{ens}}} w_z\, H_{\mathrm{upper}}(Y \mid x, Z=z).
    \label{eq:ub_aleatoric}
\end{equation}

\medskip
\noindent\textbf{Step 3: Combining the bounds.}

Substituting~\eqref{eq:lb_marginal} and~\eqref{eq:ub_aleatoric}
into~\eqref{eq:mi_app}, and using the fact that $f(a, b) = a - b$ is non-decreasing
in $a$ and non-increasing in $b$,
\begin{align}
    I(Y;\, Z \mid x)
    &= H(Y \mid x) - \mathbb{E}_Z\bigl[H(Y \mid x, Z)\bigr] \nonumber \\
    &\geq H_{\mathrm{lower}}(Y \mid x)
    - \sum_{z=1}^{n_{\mathrm{ens}}} w_z\, H_{\mathrm{upper}}(Y \mid x, Z=z)
    \;=\; \textit{MI-LB}(x). \qedhere
\end{align}

\end{proof}

\section{Experimental Details}
\label{app:experimental_details}

\subsection{Software}
Our code is implemented in JAX \cite{jax2018github} using the Flax \cite{flax2020github} and Optax \cite{deepmind2020jax-optax} libraries to define and train our neural networks. We used the \texttt{JaxMix} library \cite{guilhoto2026multimodalscientificlearningdiffusions} as a starting point for working with MDN ensembles. Our complete codebase is available at \url{https://github.com/PredictiveIntelligenceLab/JaxMix-AL}.

\subsection{Synthetic Multimodal Conditional Problem}
\label{app:exp_multimodal}

This appendix specifies the full data-generating process, model, trainer,
and active-learning protocol used in Section~\ref{sec:exp_multimodal}.
All hyperparameters listed here are the values used to produce
Figures~\ref{fig:multimodal_learning_curves} and~\ref{fig:app_multimodal_scatter}
and are exposed as defaults in
\texttt{examples/multimodal\_conditional/experiment\_config.py}.

\paragraph{Input manifold.}
Latent codes $l \in \mathbb{R}^{L}$ are drawn $l \sim \mathcal{N}(0, I_L)$
and embedded into the input space via
\[
    x = \tanh\!\bigl(A\, l + b_m\bigr),
    \qquad
    A \in \mathbb{R}^{D \times L},\;
    A_{ij} \stackrel{\mathrm{iid}}{\sim} \mathcal{N}\!\bigl(0, 1/L\bigr),
    \qquad
    b_m \sim \mathcal{N}(0, I_D).
\]
Both $A$ and $b_m$ are drawn once using
\texttt{manifold\_seed}\,$=1$ and held fixed across the benchmark; the
$\tanh$ bounds $x$ to $[-1, 1]^D$.

\paragraph{Random Fourier feature map.}
All mixture parameters are input-dependent through a fixed random Fourier
feature map
\[
    h(x) = \cos\!\bigl(\Omega\, x + \varphi\bigr) \in \mathbb{R}^{P},
    \qquad
    \Omega_{ij} \stackrel{\mathrm{iid}}{\sim} \mathcal{N}(0, 1/D),
    \quad
    \varphi_i \stackrel{\mathrm{iid}}{\sim} \mathrm{Uniform}(0, 2\pi).
\]
The parameters $(\Omega, \varphi)$ are drawn once with
\texttt{dist\_seed}\,$=42$ and held fixed.

\paragraph{Per-component means and variances.}
For each of the $K$ mixture components, uncentered means and log-variances are
\[
    \tilde{\mu}_k(x) = B_k\, h(x) + c_k,
    \qquad
    \log \Sigma_k(x) = C_k\, h(x) + b^{\mathrm{var}}_k,
\]
with $B_k, C_k \in \mathbb{R}^{M \times P}$ having i.i.d.\ $\mathcal{N}(0, 1/P)$
entries, $c_k \sim \mathcal{N}(0, c_{\mathrm{scale}}^2\, I_M)$, and
$b^{\mathrm{var}}_k = 0$. The per-component offsets $c_k$ control how far apart
the modes sit in output space. Component means are then centered to enforce
$\mathbb{E}[y \mid x] = 0$:
\[
    \mu_k(x)
    = \tilde{\mu}_k(x)
    - \sum_{j=1}^{K} \pi_j(x)\, \tilde{\mu}_j(x),
\]
and the covariance is diagonal,
$\Sigma_k(x) = \mathrm{diag}(\exp \log \Sigma_k(x))$. The variance-coupling
coefficient $\alpha$ in the code is set to $0$, so log-variances do not depend
on $\pi_k$.

\paragraph{Structured mixing weights.}
To create a clean unimodal/multimodal phase boundary, we use the
\texttt{structured} mixing mode. Letting $r = \|x_{1:L}\|$, define a radial
gate and $K{-}1$ angular scores,
\begin{align*}
    g(x) &= \tfrac{1}{2}\bigl(1 + \tanh\bigl(\beta\,(r - r_0)\bigr)\bigr), \\
    s(x) &= \mathrm{softmax}\!\bigl(\gamma\, V\, x_{1:L}\bigr) \in \Delta^{K-2},
    \qquad V \in \mathbb{R}^{(K-1) \times L},\; V_{k,:}/\|V_{k,:}\| = V_{k,:},
\end{align*}
where the rows of $V$ are unit vectors drawn from $\mathcal{N}(0, I_L)$ and
normalised, and then assemble the logits
\[
    \ell_0(x) = \mathrm{scale}\cdot(1 - g(x)),
    \qquad
    \ell_k(x) = \mathrm{scale}\cdot g(x)\cdot s_k(x)\;\; (k = 1, \dots, K{-}1),
\]
and set $\pi(x) = \mathrm{softmax}(\ell(x))$. Inside the radius ($r < r_0$)
the mass concentrates on component $0$, making $p^*(y \mid x)$ effectively
unimodal; outside the radius the angular scores activate components
$1, \dots, K{-}1$ in sectors.

\paragraph{Hyperparameters.}
Table~\ref{tab:app_multimodal_dist_hparams} lists the distribution
hyperparameters and Table~\ref{tab:app_multimodal_model_hparams} the model and
active-learning hyperparameters. All values match the defaults in
\texttt{experiment\_config.py}.

\begin{table}[htbp]
\centering
\caption{Distribution hyperparameters for the multimodal conditional
problem.}
\label{tab:app_multimodal_dist_hparams}
\begin{tabular}{lll}
\toprule
Symbol & Value & Meaning \\
\midrule
$D$ & $10$ & Input dimension \\
$M$ & $16$ & Output dimension \\
$L$ & $4$ & Latent manifold dimension \\
$K$ & $3$ & Number of true mixture components \\
$P$ & $128$ & Random Fourier features \\
$c_{\mathrm{scale}}$ & $10.0$ & Std.\ of per-component offsets $c_k$ \\
$\alpha$ & $0.0$ & Variance-coupling coefficient (disabled) \\
$\tau$ & $1.0$ & Softmax temperature (unused in structured mode) \\
$\beta$ & $8.0$ & Transition sharpness \\
$r_0$ & $1.3$ & Transition radius \\
$\gamma$ & $2.0$ & Angular sharpness \\
$\mathrm{scale}$ & $3.0$ & Logit magnitude \\
\texttt{dist\_seed} & $42$ & PRNG seed for distribution parameters \\
\texttt{manifold\_seed} & $1$ & PRNG seed for manifold map $(A, b_m)$ \\
\bottomrule
\end{tabular}
\end{table}

\paragraph{Oracle NLL.}
Because $p^*(y \mid x)$ is known in closed form, the oracle negative
log-likelihood
$\mathrm{NLL}^{*} = -\mathbb{E}_{(x, y) \sim p^*}[\log p^*(y \mid x)]$ can be
evaluated exactly on the test set. Under the settings above it equals
$\mathrm{NLL}^{*} = 22.98$ (computed by
\texttt{compute\_true\_nll} in
\texttt{examples/multimodal\_conditional/utils.py}).

\paragraph{Model architecture.}
Each ensemble member is a Mixture Density Network with shared backbone:
a $2$-layer MLP with $128$ hidden units per layer and GELU
activations~\cite{hendrycks2016gelu}. The MLP output is linearly projected to
$K_{\mathrm{MDN}} = 5$ mixture components, each with a mean in
$\mathbb{R}^{M}$ and a diagonal covariance; mixture weights come from a
softmax head. Ensemble size is $n_{\mathrm{ens}} = 8$, with members
initialised from different PRNG seeds and trained independently on the same
labeled set (the standard deep-ensemble construction used to instantiate
$P_Z = \mathrm{Uniform}\{1, \dots, n_{\mathrm{ens}}\}$).

\paragraph{Training schedule.}
All members are trained with AdamW (peak learning rate
$5\times 10^{-4}$, weight decay $10^{-2}$) preceded by adaptive
gradient clipping at $0.1$. The learning rate follows a linear warmup
over $\min(500,\, n_{\mathrm{iter}}/5)$ steps to the peak value, after
which it decays exponentially at rate $0.9$ per $2{,}000$ steps. The
number of gradient steps per active learning round is set adaptively
to
\[
    n_{\mathrm{iter}}(n_{\mathrm{lab}})
    = \min\!\bigl(10{,}000,\; 10 \cdot n_{\mathrm{lab}}\bigr),
\]
where $n_{\mathrm{lab}}$ is the current labeled-set size; this prevents
overfitting on small labeled sets in early rounds while allowing full
optimisation once $n_{\mathrm{lab}} \geq 1000$. The training mini-batch
size is $128$ (full-batch when $n_{\mathrm{lab}} < 128$). Optimiser,
architecture, and active-learning hyperparameters are fixed across all
acquisition functions; we do not perform per-method tuning.

\paragraph{Active-learning protocol.}
For each of $5$ seeds we instantiate a fresh pool of $50{,}000$ candidate
inputs, a held-out test set of $2{,}000$ inputs, and an initial labeled set
of $100$ inputs. Acquisition scores are evaluated on the remaining pool in
chunks of $256$. Each active-learning run performs $20$ rounds, acquiring
$50$ queries per round, so the final labeled set contains
$100 + 20 \cdot 50 = 1100$ examples, i.e.\ $2.2\%$ of the pool. Test NLL is
evaluated on the held-out set at the end of every round.

\begin{table}[htbp]
\centering
\caption{Model and active-learning hyperparameters for the multimodal
conditional problem.}
\label{tab:app_multimodal_model_hparams}
\begin{tabular}{lll}
\toprule
Symbol / option & Value & Meaning \\
\midrule
$n_{\mathrm{ens}}$ & $8$ & Ensemble size \\
$K_{\mathrm{MDN}}$ & $5$ & MDN mixture components per member \\
hidden features & $128$ & MLP width \\
depth & $2$ & Number of MLP hidden layers \\
activation & GELU & Nonlinearity \\
training batch size & $128$ & Mini-batch size for AdamW \\
peak learning rate & $5\times 10^{-4}$ & AdamW peak LR (after warmup) \\
weight decay & $10^{-2}$ & AdamW weight decay \\
gradient clip & $0.1$ & Adaptive gradient clipping threshold \\
$n_{\mathrm{iter}}$ (cap) & $10{,}000$ & Max gradient steps per round \\
iter-per-sample & $10$ & Slope of adaptive schedule \\
candidate pool size & $50{,}000$ & Unlabeled pool $|\mathcal{X}_{\mathrm{pool}}|$ \\
test set size & $2{,}000$ & Held-out evaluation set \\
initial labelled & $100$ & Initial labeled budget \\
AL rounds & $20$ & Number of acquisition rounds \\
query batch size & $50$ & Queries acquired per round \\
acquisition batch size & $256$ & Chunk size for scoring the pool \\
seeds & $\{0,1,2,3,4\}$ & Data / training seeds for reported curves \\
\bottomrule
\end{tabular}
\end{table}

\paragraph{Acquisition functions.}
Five scoring functions are evaluated. \textbf{Random} draws scores i.i.d.\ $\mathrm{Uniform}(0,1)$. \textbf{Epistemic Variance} uses the trace of the covariance of conditional means across ensemble members. \textbf{MI-LB} is~\eqref{eq:gee_explicit}. \textbf{BAIT}~\cite{ash2021gone} uses the paper's last-layer mean-head Fisher of ensemble member $0$ with one MC sample $y \sim p_\theta(\cdot \mid x)$ per pool point. \textbf{Core-Set}~\cite{sener2018active} runs k-Center-Greedy on the shared MDN backbone activations of ensemble member $0$ (the ``final FC layer'' recipe of \S 4.4). The latter two produce no scalar score and bypass the score-to-batch step.

\paragraph{Implementation of BAIT and Core-Set in the MDN setting.}
Both baselines are adapted from their original classification specifications. \textbf{BAIT} requires per-input Fisher embeddings $G(x)$ such that $F(x) = G(x)^\top G(x)$ approximates the per-sample Fisher information; following the last-layer recipe of \cite{ash2021gone} we restrict the Fisher to the mean-head weights of ensemble member $0$. Drawing one Monte-Carlo sample $y \sim p_\theta(\cdot \mid x)$ and using the closed-form MDN log-likelihood gradient $\nabla_{W_\mu} \log p(y \mid x; \theta) = \gamma_k(y, x)\,(y - \mu_k(x))/\sigma_k^2(x) \otimes z(x)$, where $\gamma_k$ are posterior mixture responsibilities and $z(x)$ is the shared backbone activation, the embedding is the flattened outer product, $G(x) \in \mathbb{R}^{1 \times hKd}$. Selection minimises $\operatorname{tr}((F_{\mathrm{train}} + \lambda I)^{-1} F_{\mathrm{cand}})$ via the forward$+$backward greedy of \cite{ash2021gone} with ridge $\lambda = 10^{-3}$ and the labelled-set Fisher as the burden matrix. This single-ensemble, single-MC choice is a compute trade-off rather than a capacity claim: extending the Fisher embedding to all $n_{\mathrm{ens}} = 8$ members (to capture cross-member disagreement and the mixture-weight signal the mean-head Fisher omits) or to multiple MC samples per pool point would multiply an already dominant selection overhead --- BAIT alone consumes $\sim\!12$ of the $\sim\!25$ GPU-h total compute budget (Table~\ref{tab:app_compute_runtimes}), pushing a richer embedding into a $100$+ GPU-h regime no other acquisition here requires. We therefore report the canonical last-layer recipe of~\cite{ash2021gone} and flag the resulting Fisher-estimate variance as a known limitation in the empirical discussion below. \textbf{Core-Set} uses the activations of the shared MDN backbone of ensemble member $0$ (the layer immediately before the mixture head) as $h$-dimensional features, then runs k-Center-Greedy~\cite{sener2018active} on Euclidean distances initialised at the labelled set. Both methods consume the full pool without subsampling.

\paragraph{Selection strategies.}
Three batch-selection rules convert per-point scores into a query batch of
size $50$:
\begin{itemize}
    \item \textbf{Top-$k$ (greedy).} Select the $50$ points with the highest
    acquisition scores. Used for all main-text results.
    \item \textbf{SBAL~\cite{kirsch2022sbal}.} We use the \emph{softmax}
    variant of SBAL (as opposed to the softrank and power variants in
    Kirsch et al.), which handles negative scores without modification:
    sample $50$ points without replacement from
    $\mathrm{softmax}(\mathrm{score}/T)$ via the Gumbel-top-$k$ trick.
    The temperature $T$ interpolates between exploitation ($T \to 0$,
    recovers top-$k$) and uniform exploration ($T \to \infty$, recovers
    Random). All SBAL runs reported here use $T = 1.0$.
    \item \textbf{MaxDist~\cite{holzmuller2023framework}.} Acquisition-weighted
    farthest-point sampling in feature space (LCMD-TP variant). Given input
    features $x$ standardised per dimension and normalised scores
    $\tilde{s} \in [0, 1]$, greedily pick the point maximising
    $d_{\min}(i)\cdot(1 + w\, \tilde{s}(i))$, where $d_{\min}(i)$ is the
    squared distance from point $i$ to its nearest already-selected or
    already-training point. All MaxDist runs use score-weight $w = 1$. At
    $w = 0$ this reduces to pure farthest-point sampling, and as
    $w \to \infty$ it approaches top-$k$ on $\tilde{s}$.
\end{itemize}
SBAL and MaxDist only apply to the deterministic acquisitions (Epistemic
Variance and MI-LB); combining either with Random is undefined (Random has no
meaningful score ranking or signal to weight against). Any acquisition of the
form \texttt{sbal\_X} or \texttt{maxdist\_X} in the experiment log corresponds
to applying the named selection strategy on top of base score $X$. BAIT and Core-Set bypass the score-to-batch step entirely (Fisher-trace forward+backward greedy and k-Center-Greedy respectively) and have no SBAL/MaxDist variants.

\paragraph{Results for SBAL and MaxDist variants.}
Figure~\ref{fig:app_multimodal_batch_variants_curves} and
Table~\ref{tab:app_multimodal_final_nll} report final test NLL at
$n = 1100$ across $5$ seeds. At $T = 1.0$, MI-LB (SBAL) sits roughly halfway
between MI-LB (top-$k$) and Random ($33.99$ vs $31.12$ and $39.73$), because
softmax sampling at $T = 1.0$ places non-trivial mass outside the
high-score region near the radial gate; lowering $T$ tightens the
distribution back onto MI-LB (top-$k$) and recovers it exactly as $T \to 0$.
MI-LB (MaxDist) at $w = 1$ matches MI-LB (top-$k$) ($30.49$ vs $31.12$),
adding batch diversity without sacrificing score quality; larger $w$
moves it toward top-$k$ on the score, smaller $w$ toward pure
farthest-point sampling.

\paragraph{BAIT and Core-Set.}
Core-Set ties MI-LB within seed-to-seed noise ($30.57 \pm 0.56$ vs $31.12 \pm 0.34$); MI-LB retains the tighter spread. The tie reflects benchmark geometry: pool inputs lie on a $4$-D tanh manifold, so k-Center-Greedy spreads queries across that manifold and ends up sampling the gate region that MI-LB targets through entropy disagreement. BAIT effectively ties Variance ($32.26 \pm 1.55$ vs $32.56 \pm 1.14$), but with the largest seed-to-seed spread of any acquisition we evaluate --- the variance of the single-MC Fisher estimate at small budgets. Both baselines bypass the Two-Index decomposition; the next two appendices test whether the apparent tie generalises beyond a benchmark in which input geometry already encodes the multimodality.

\begin{table}[htbp]
\centering
\small
\caption{Final test NLL at $n = 1100$, mean $\pm$ std across $5$ seeds
(min--max in brackets). SBAL uses temperature $T = 1.0$; MaxDist uses
score weight $w = 1$. Oracle $\mathrm{NLL}^\star = 22.98$.}
\label{tab:app_multimodal_final_nll}
\begin{tabular}{lcc}
\toprule
Acquisition & Mean $\pm$ std & Min -- Max \\
\midrule
Random                       & $39.73 \pm 0.90$ & $38.89$ -- $40.97$ \\
Epistemic Variance (top-$k$) & $32.56 \pm 1.14$ & $31.41$ -- $34.14$ \\
Epistemic Variance (SBAL)    & $31.01 \pm 0.57$ & $30.43$ -- $31.61$ \\
MI-LB (top-$k$)                & $31.12 \pm 0.34$ & $30.58$ -- $31.37$ \\
MI-LB (SBAL)                   & $33.99 \pm 0.74$ & $33.21$ -- $35.01$ \\
MI-LB (MaxDist)                & $30.49 \pm 0.46$ & $29.84$ -- $30.91$ \\
BAIT                         & $32.26 \pm 1.55$ & $30.40$ -- $34.34$ \\
Core-Set                     & $\mathbf{30.57 \pm 0.56}$ & $29.97$ -- $31.31$ \\

\bottomrule
\end{tabular}
\end{table}

\begin{figure}[htbp]
    \centering
    \includegraphics[width=0.5\textwidth]{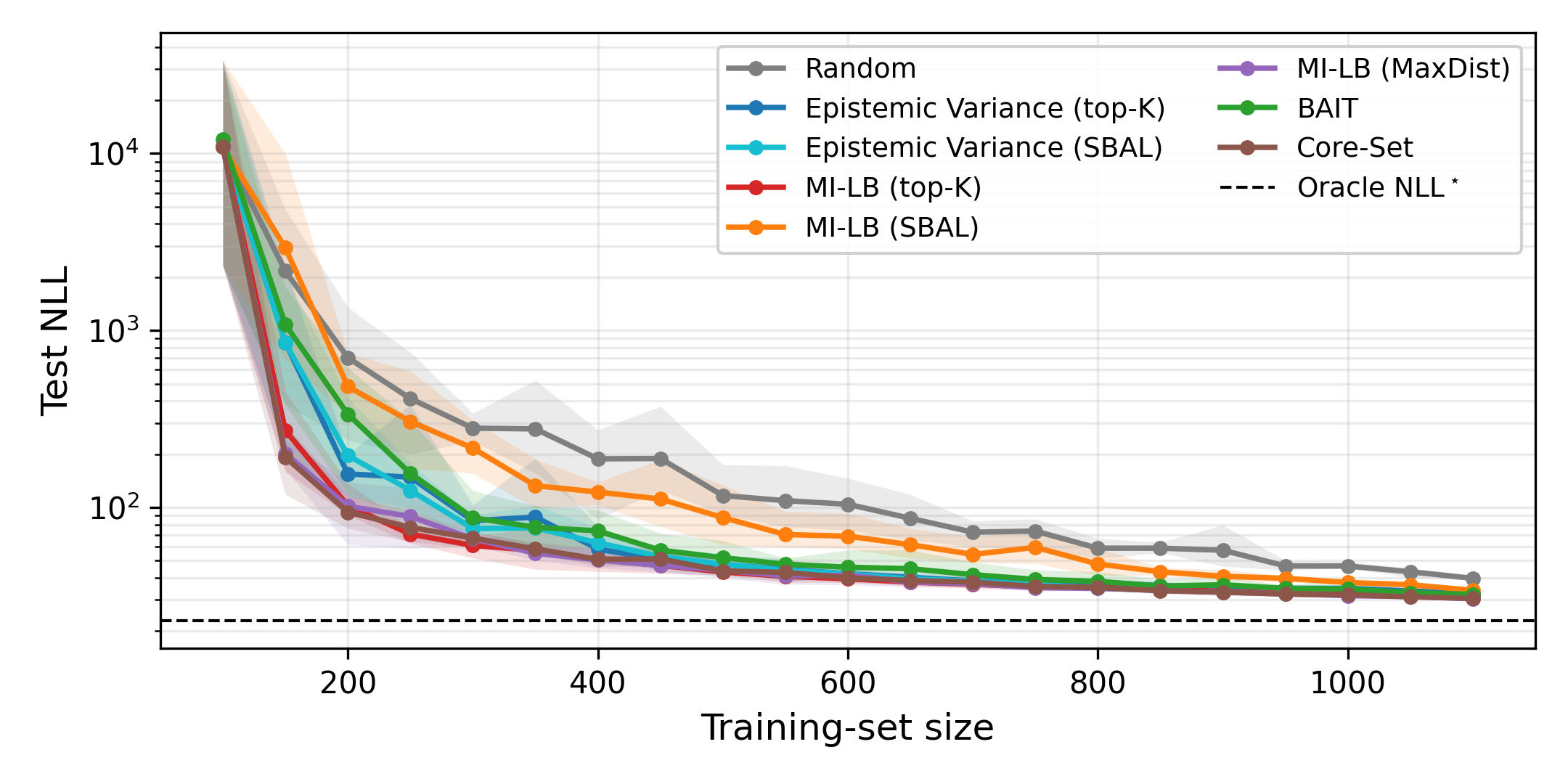}
    \caption{
    Learning curves on the multimodal benchmark for SBAL ($T=1.0$) and MaxDist ($w=1$) variants of Variance and MI-LB, plus the BAIT and Core-Set baselines, alongside the three top-$k$ curves from Fig.~\ref{fig:multimodal_learning_curves} ($5$ seeds; bands min--max).
    }
    \label{fig:app_multimodal_batch_variants_curves}
\end{figure}

\begin{figure}[htbp]
    \centering
    \includegraphics[width=\textwidth]{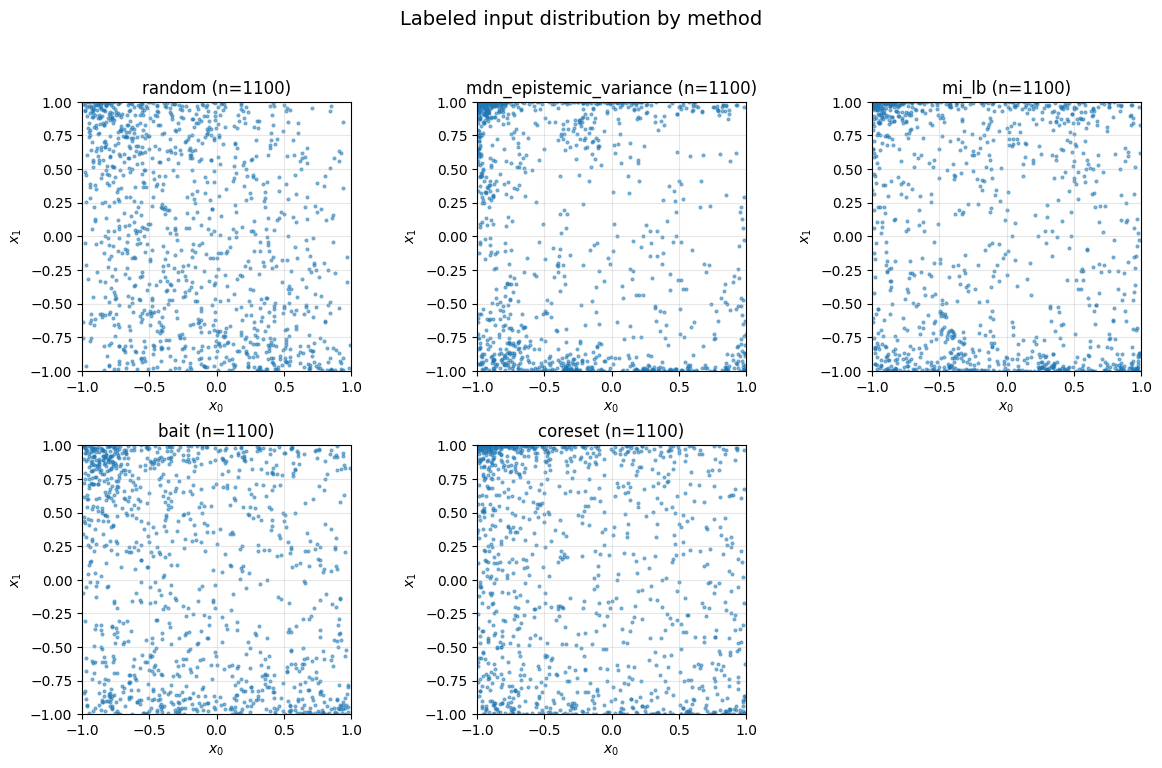}
    \caption{
Spatial distribution of all labeled inputs (initial + acquired) at
$n = 1100$ projected onto $(x_0, x_1)$ for the five base
acquisitions: Random, Epistemic Variance, and MI-LB
\textbf{(top row)}, and the BAIT and Core-Set baselines
\textbf{(bottom row)}. By the end of the budget the four informative
acquisitions look broadly similar --- all concentrate mass along the
multimodal edge of the pool ($x_0 \to -1$) and the upper/lower
boundary --- while Random remains the clear outlier, spreading
queries uniformly across the support. All five panels inherit the
same off-center pool support from the fixed bias $b$ in
$x = \tanh(Al + b)$ (drawn once at benchmark instantiation; see
\emph{Input manifold} above).
}
    \label{fig:app_multimodal_scatter}
\end{figure}

\subsection{Coupled Double-Well System}
\label{app:exp_coupled_double_well}

This appendix specifies the full data-generating process, model, trainer,
and active-learning protocol used in
Section~\ref{sec:exp_coupled_double_well}. 
\paragraph{Simulator.}
$P = 5$ particles evolve according to the overdamped Langevin
SDE~\eqref{eq:dw_sde} under open boundary conditions: particle $1$
couples only to particle $2$, particle $P$ only to particle $P{-}1$, and
interior particles to both neighbours. Barrier heights are uniform,
$a = 1$, giving per-particle potential $V(q) = q^4/4 - q^2/2$ with
minima at $q = \pm 1$ and barrier height $a/4 = 0.25$. We integrate
with the Euler--Maruyama scheme
\[
    q_i^{(t+dt)}
    = q_i^{(t)}
    + \bigl[a\,(q_i^{(t)} - (q_i^{(t)})^3)
        + \kappa \!\sum_{j \in \mathrm{nn}(i)}\! (q_j^{(t)} - q_i^{(t)})\bigr]\, dt
    + \sigma\, \sqrt{dt}\, \eta_i^{(t)},
    \qquad \eta_i^{(t)} \sim \mathcal{N}(0, 1),
\]
with $dt = 0.005$ and terminal time $T = 5.0$ (i.e.\ $n_{\mathrm{steps}}
= 1000$).

\paragraph{Inputs and outputs.}
Each trajectory is summarised by the $7$-dimensional input
$x = (q_1(0), \dots, q_5(0),\, \sigma,\, \kappa) \in \mathbb{R}^{P+2}$,
drawn component-wise from
\[
    q_i(0) \sim \mathrm{Uniform}(-1.5, 1.5),\quad
    \sigma \sim \mathrm{Uniform}(0.3, 2.0),\quad
    \kappa \sim \mathrm{Uniform}(0, 3.0).
\]
The output is the concatenation of $n_{\mathrm{snap}} = 4$ snapshots of
the particle positions evenly spaced in $[T/n_{\mathrm{snap}}, T]$ (i.e.\
at $t = 1.25, 2.5, 3.75, 5.0$), flattened into
$y \in \mathbb{R}^{n_{\mathrm{snap}} \cdot P} = \mathbb{R}^{20}$.
Multi-snapshot outputs give the model direct access to trajectory-level
structure and substantially improve training stability compared to
single-endpoint outputs, at no additional simulation cost.

\paragraph{Hyperparameters.}
Table~\ref{tab:app_dw_phys_hparams} lists the simulator and data
hyperparameters and Table~\ref{tab:app_dw_model_hparams} the model and
active-learning hyperparameters.

\begin{table}[htbp]
\centering
\caption{Simulator and data hyperparameters for the coupled double-well
benchmark.}
\label{tab:app_dw_phys_hparams}
\begin{tabular}{lll}
\toprule
Symbol / option & Value & Meaning \\
\midrule
$P$ & $5$ & Number of coupled particles \\
$a$ & $1$ (uniform) & Barrier-height coefficient \\
$T$ & $5.0$ & Terminal integration time \\
$dt$ & $0.005$ & Euler--Maruyama step ($1000$ steps) \\
$n_{\mathrm{snap}}$ & $4$ & Snapshots recorded per trajectory \\
boundary & open & Chain ends couple only to one neighbour \\
$q_i(0)$ range & $[-1.5,\, 1.5]$ & Initial-configuration prior \\
$\sigma$ range & $[0.3,\, 2.0]$ & Noise-intensity prior \\
$\kappa$ range & $[0,\, 3.0]$ & Coupling-strength prior \\
\bottomrule
\end{tabular}
\end{table}

\paragraph{Model architecture.}
Each ensemble member is an MDN with a $3$-hidden-layer MLP backbone,
$128$ hidden units per layer, GELU activations, and a final linear
projection to $K_{\mathrm{MDN}} = 8$ mixture components with diagonal
covariances over the $20$-dimensional output. The ensemble size is
$n_{\mathrm{ens}} = 8$, with members initialised from different PRNG
seeds and trained independently on the same labeled set.

\paragraph{Training schedule.}
All members are trained with AdamW using the same optimiser
configuration as in Appendix~\ref{app:exp_multimodal}: peak learning
rate $5\times 10^{-4}$, weight decay $10^{-2}$, linear warmup followed
by exponential decay at rate $0.9$ per $2{,}000$ steps, and adaptive
gradient clipping at $0.1$. The number of gradient steps per round is
set adaptively to $n_{\mathrm{iter}}(n_{\mathrm{lab}}) = \min(10{,}000,\;
10 \cdot n_{\mathrm{lab}})$, with mini-batch size $128$.

\paragraph{Active-learning protocol.}
For each of $5$ seeds we instantiate a fresh pool of $50{,}000$
candidate inputs, a held-out test set of $2{,}000$ inputs, and an
initial labeled set of $100$ inputs. Acquisition scores are evaluated
on the remaining pool in chunks of $256$. Each run performs $20$
rounds of $50$ queries each, so the final labeled set contains
$100 + 20 \cdot 50 = 1{,}100$ trajectories ($2.2\%$ of the pool). Test
NLL is evaluated at the end of every round. Seeds $\{0, 1, 2, 3, 4\}$
are used for all reported curves.

\begin{table}[htbp]
\centering
\caption{Model and active-learning hyperparameters for the coupled
double-well benchmark.}
\label{tab:app_dw_model_hparams}
\begin{tabular}{lll}
\toprule
Symbol / option & Value & Meaning \\
\midrule
$n_{\mathrm{ens}}$ & $8$ & Ensemble size \\
$K_{\mathrm{MDN}}$ & $8$ & MDN mixture components per member \\
hidden features & $128$ & MLP width \\
depth & $3$ & Number of MLP hidden layers \\
activation & GELU & Nonlinearity \\
training batch size & $128$ & Mini-batch size for AdamW \\
peak learning rate & $5\times 10^{-4}$ & AdamW peak LR (after warmup) \\
weight decay & $10^{-2}$ & AdamW weight decay \\
gradient clip & $0.1$ & Adaptive gradient clipping threshold \\
$n_{\mathrm{iter}}$ (cap) & $10{,}000$ & Max gradient steps per round \\
iter-per-sample & $10$ & Slope of adaptive schedule \\
candidate pool size & $50{,}000$ & Unlabeled pool $|\mathcal{X}_{\mathrm{pool}}|$ \\
test set size & $2{,}000$ & Held-out evaluation set \\
initial labelled & $100$ & Initial labeled budget \\
AL rounds & $20$ & Number of acquisition rounds \\
query batch size & $50$ & Queries acquired per round \\
acquisition batch size & $256$ & Chunk size for scoring the pool \\
seeds & $\{0,1,2,3,4\}$ & Data / training seeds \\
\bottomrule
\end{tabular}
\end{table}

\paragraph{Acquisition functions and selection strategies.}
Five base acquisitions, defined in Appendix~\ref{app:exp_multimodal} unchanged: Random, Variance, MI-LB, BAIT~\cite{ash2021gone}, Core-Set~\cite{sener2018active}. Selection strategies: top-$k$ for Random / Variance / MI-LB; SBAL ($T = 1.0$) for Variance and MI-LB; MaxDist ($w = 1$) for MI-LB. BAIT and Core-Set bypass the score-to-batch step entirely.
\paragraph{Results for SBAL and MaxDist variants.}
Figure~\ref{fig:app_dw_batch_variants_curves} and
Table~\ref{tab:app_dw_final_nll} report the full learning curves and
the final test NLL at $n = 1100$ across $5$ seeds. Two patterns are
worth highlighting.
First, MaxDist pairs well with MI-LB: MI-LB (MaxDist) at $w = 1$ finishes
at $131 \pm 15$, within a factor of two of MI-LB (top-$k$) and
comparable to Epistemic Variance (top-$k$). Acquisition-weighted
farthest-point sampling adds batch diversity without discarding the
informative ranking.
Second, SBAL severely degrades both base scores on this benchmark. MI-LB (SBAL) ends at $553 \pm 92$ and Epistemic Variance (SBAL) at $\mathbf{457 \pm 46}$, both within striking distance of the Random baseline ($518 \pm 60$) and roughly $4$--$8\times$ worse than their respective top-$k$ counterparts ($122$ and $71$). This is the opposite of the multimodal-conditional result in
Appendix~\ref{app:exp_multimodal}, where SBAL was competitive or
helpful, and it reflects a quantitative rather than a qualitative
difference: the total budget ($1{,}100$ points) is small relative to
the pool, so softmax sampling at $T = 1.0$ spreads enough mass onto
low-score points that the effective fraction of informative queries
collapses toward the Random baseline; annealing $T$ toward zero would
tighten the distribution onto the top-$k$ result and recover it in the
limit $T \to 0$. This is consistent with the original observations of
Kirsch et al.~\cite{kirsch2022sbal}: stochastic batch acquisition helps
when the top of the score ranking is corrupted by approximation noise
and diverse sampling hedges against that noise; it hurts when the top
ranking is already approximately correct and exploration trades
informative points for uninformative ones.

\paragraph{BAIT and Core-Set.}
Both geometric baselines fail decisively on this benchmark: Core-Set finishes at $304.3 \pm 54.8$ and BAIT at $281.3 \pm 15.1$, $\sim\!4\times$ worse than MI-LB ($70.8 \pm 8.0$) and within $1.7\times$ of Random ($518.1 \pm 60.1$). Mode structure here lives in \emph{output} space (Kramers escape between $q = \pm 1$): k-Center-Greedy spreads queries uniformly over $(\sigma, \kappa)$ regardless of where the bimodal regime sits, and the single-MC last-layer Fisher embedding is dominated by within-mode noise. Their joint collapse is the strongest evidence in the paper that MI-LB's advantage is not subsumed by feature-space coverage or last-layer information geometry.

\begin{table}[htbp]
\centering
\small
\caption{Final test NLL at $n = 1100$, mean $\pm$ std across $5$ seeds
(min--max in brackets) for the coupled double-well benchmark. All SBAL
runs use $T = 1.0$; MaxDist uses score weight $w = 1$.}
\label{tab:app_dw_final_nll}
\begin{tabular}{lcc}
\toprule
Acquisition & Mean $\pm$ std & Min -- Max \\
\midrule
Random                       & $518.1 \pm 60.1$ & $462.6$ -- $616.2$ \\
Epistemic Variance (top-$k$) & $122.1 \pm 11.3$ & $109.1$ -- $138.5$ \\
Epistemic Variance (SBAL)    & $456.8 \pm 45.7$ & $392.9$ -- $515.3$ \\
MI-LB (top-$k$)                & $\mathbf{70.8 \pm 8.0}$  & $59.9$  -- $81.6$  \\
MI-LB (SBAL)                   & $553.0 \pm 91.9$ & $423.5$ -- $630.2$ \\
MI-LB (MaxDist)                & $130.7 \pm 14.9$ & $113.5$ -- $147.0$ \\
BAIT                         & $281.3 \pm 15.1$ & $268.8$ -- $305.3$ \\
Core-Set                     & $304.3 \pm 54.8$ & $249.3$ -- $377.5$ \\
\bottomrule
\end{tabular}
\end{table}

\begin{figure}[htbp]
    \centering
    \includegraphics[width=0.85\textwidth]{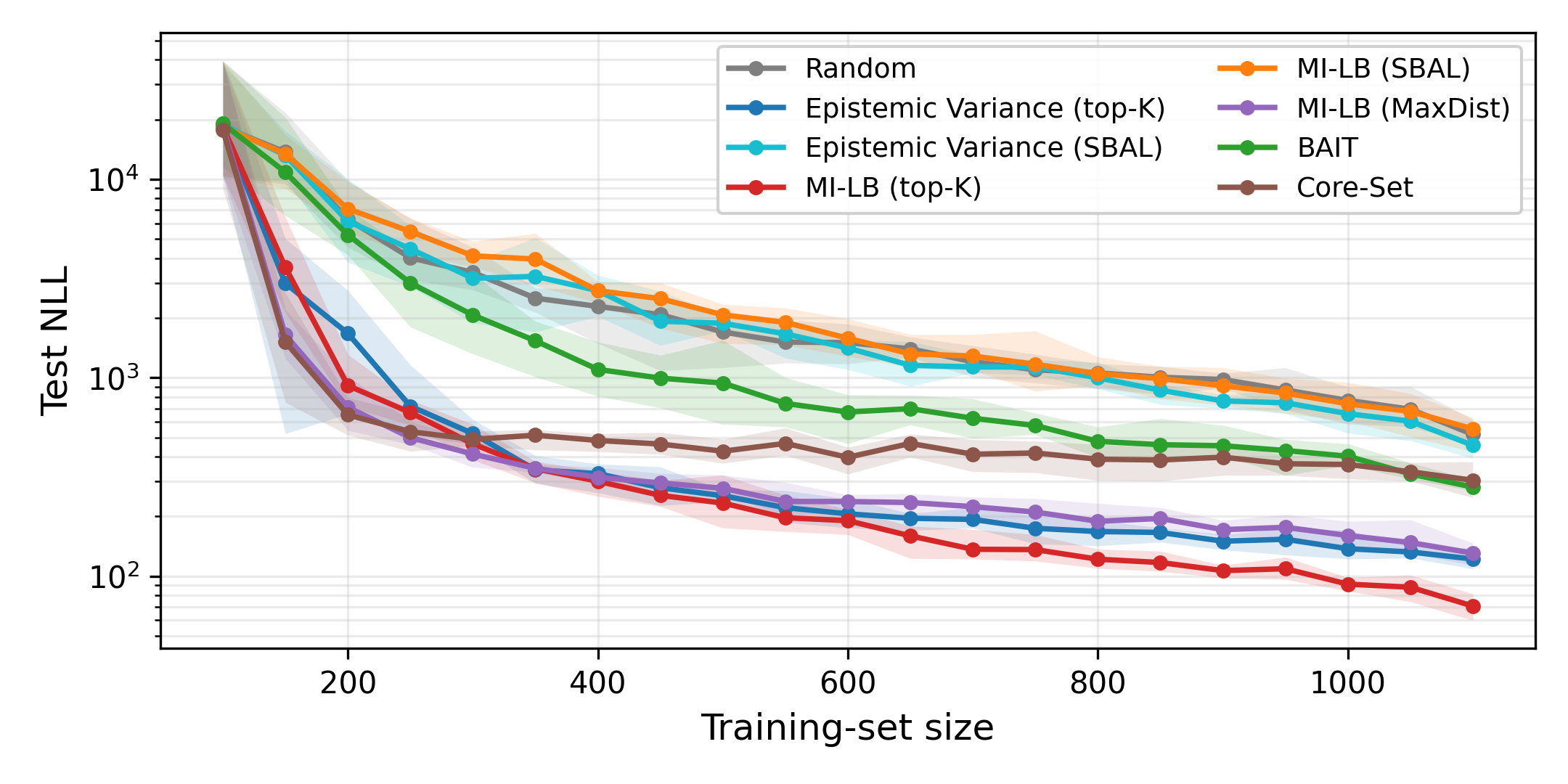}
    \caption{
    Learning curves for all eight (acquisition, selection-strategy) combinations on the coupled double-well benchmark including the BAIT and Core-Set baselines ($5$ seeds; bands min--max). MaxDist-MI-LB stays within $\sim\!2\times$ of top-$k$ MI-LB; both SBAL variants and the two geometric baselines (BAIT, Core-Set) collapse toward the Random regime.
    }
    \label{fig:app_dw_batch_variants_curves}
\end{figure}

\paragraph{Why a mixture head: $K = 1$ vs $K = 8$.}
To justify the choice of a mixture head
($K_{\mathrm{MDN}} = 8$) over a single-Gaussian head ($K = 1$), we
train both with the identical architecture used in the AL experiments
(depth-$3$ MLP, $128$ hidden units, $n_{\mathrm{ens}} = 8$) on a
larger offline dataset of $200{,}100$ trajectories drawn from the same
input distribution as the active-learning experiments, and evaluate
both on a held-out test set of $5{,}000$ trajectories. The
ensemble-averaged test NLL is $\mathbf{15.90}$ for $K = 8$ versus
$\mathbf{21.41}$ for $K = 1$, a gap of $5.51$ nats on the $20$-%
dimensional output (roughly $0.28$ nats per output dimension). Absolute
NLL values here are below those in the AL experiments because training
uses $200{,}100$ examples rather than the $1{,}100$ queried by the AL
loop; the quantity of interest is the $K = 8$-vs-$K = 1$ gap, which is
driven by the head structure, not the training-set size. The $K = 1$
model must place a single Gaussian per input, so wherever the
conditional is bimodal it is forced to smear mass between the wells
and pay an unavoidable log-likelihood penalty; the $K = 8$ mixture can
allocate one component per well.
Figure~\ref{fig:app_dw_mdn_baseline} shows the qualitative picture on
the test distribution: ground-truth samples of $(q_0(T), q_1(T))$
concentrate near $(\pm 1, \pm 1)$ in the bimodal regime and the $K = 8$
MDN recovers the same multi-well envelope, while the $K = 1$ model
collapses to a single broad ellipse near the origin. This is the
failure mode anticipated in
Section~\ref{sec:exp_coupled_double_well} and the quantitative reason
the active-learning experiments on this benchmark use
$K_{\mathrm{MDN}} = 8$ throughout.

\begin{figure}[htbp]
    \centering
    \includegraphics[width=0.95\textwidth]{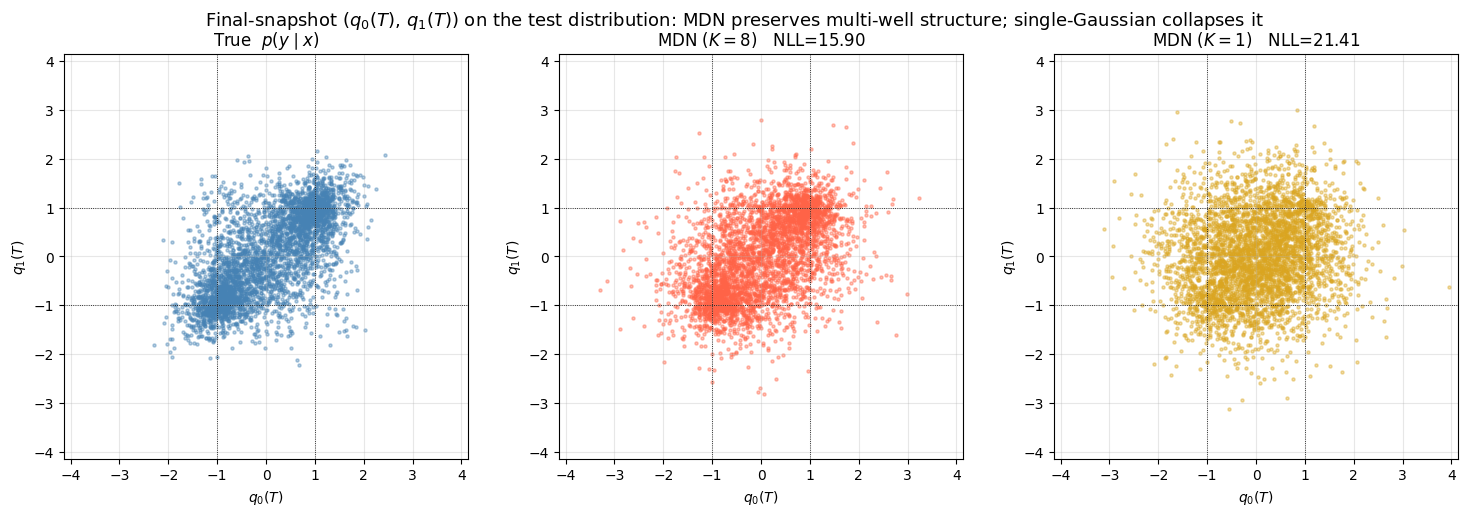}
    \caption{
    Coupled double-well benchmark: final-snapshot positions
    $(q_0(T), q_1(T))$ for $4{,}000$ held-out test inputs drawn from
    the same joint distribution as the AL test set.
    \textbf{Left:} ground-truth simulator samples; mass concentrates
    near $(\pm 1, \pm 1)$, the signature of the per-particle bimodal
    regime at $\sigma \gtrsim \sqrt{a/2} \approx 0.71$, with a
    diagonal enhancement that reflects neighbour alignment under
    non-zero $\kappa$.
    \textbf{Middle:} samples from the $K = 8$ MDN ensemble (net $0$);
    the mixture head preserves the same multi-well envelope (test
    NLL $15.90$).
    \textbf{Right:} samples from a $K = 1$ single-Gaussian MDN trained
    with the same architecture and budget; per-input unimodality
    forces mass into a single broad ellipse centred near the origin
    (test NLL $21.41$; $+5.51$ nats over $K = 8$). Dotted lines mark
    the well locations $q = \pm 1$.
    }
    \label{fig:app_dw_mdn_baseline}
\end{figure}

\subsection{Material Science Application: Alloy Phase Competition}
\label{app:exp_ternary_phases}

This appendix details the synthetic phase-competition simulator and the
active-learning protocol summarised in
Section~\ref{sec:exp_ternary_phases}, and reports the SBAL and MaxDist
batch-diversity variants deferred from the main text.

\paragraph{Phase model.}
The simulator implements a softmin-over-quadratic-free-energies model
inspired by the CALPHAD framework~\cite{lukas2007calphad}. For each of
$N_\phi = 4$ nominal phases $\phi$ we draw a positive-definite
$3 \times 3$ Hessian as
$H_\phi = R_\phi R_\phi^\top + 0.3 \cdot I_3$
with $R_\phi$ a Gaussian matrix ($R_{\phi,ij} \sim \mathcal{N}(0, 0.5^2)$),
and a linear bias $b_\phi \in \mathbb{R}^3$ with
$b_{\phi,i} \sim \mathcal{N}(0, 2^2)$. Writing $x_3 = (x_A, x_B, x_C)$ for
the composition coordinates extended onto the $3$-simplex, the Gibbs
free energy of phase $\phi$ is the quadratic form
\[
    G_\phi(x_3) = \tfrac{1}{2}\, x_3^\top H_\phi\, x_3 + b_\phi^\top x_3,
\]
and the phase posterior at temperature $\tau_G$ is
$\pi(\phi \mid x_3) \propto \exp(-G_\phi(x_3) / \tau_G)$. This posterior
does not depend on the process parameters $p \in \mathbb{R}^{n_{\mathrm{proc}}}$;
phase assignment is a function of composition alone.

The conditional response mean for phase $\phi$ is
\[
    \mu_\phi(x_3, p) = c_\phi^\top x_3 + d_\phi
        + \tfrac{1}{2}\sin(\omega_\phi^\top x_3)
        + W_\phi^\top p,
\]
with $c_\phi \in \mathbb{R}^3$, $d_\phi \in \mathbb{R}$,
$\omega_\phi \in \mathbb{R}^3$, and
$W_\phi \in \mathbb{R}^{n_{\mathrm{proc}}}$ all drawn from
independent Gaussians (entry-wise scales
$c_{\mu\text{-scale}} = 6$ for $c_\phi$,
$2$ for $d_\phi$ and $\omega_\phi$,
$1$ for $W_\phi$). The per-phase log-variance is the affine function
$\log \sigma_\phi^2(x_3) = e_\phi^\top x_3 + f_\phi$ with
$e_{\phi,i} \sim \mathcal{N}(0, 0.5^2)$ and
$f_\phi \sim \mathcal{N}(-1, 0.3^2)$, so per-phase noise scales vary
smoothly with composition around a typical $\sigma_\phi \approx 0.6$.

The output is sampled by first drawing $\phi \sim \pi(\cdot \mid x_3)$,
then drawing
$y \mid \phi, x_3, p \sim \mathcal{N}(\mu_\phi(x_3, p),\, \sigma_\phi^2(x_3))$;
the ground-truth conditional $p^\star(y \mid x)$ is therefore a Gaussian
mixture in $y$ with composition-dependent mixing weights, composition-
and-process-dependent component means, and composition-dependent
component variances. The number of components with appreciable weight
at any given $x$ is determined by the realised phase structure of the
chosen system seed, summarised below.

\paragraph{Inputs and outputs.}
The input is the $8$-dimensional vector
$x = (x_A, x_B, p_1, \dots, p_6)$ with $(x_A, x_B, x_C)$ drawn from a
symmetric Dirichlet$(\alpha = (1, 1, 1))$ on the $3$-simplex and
$p_i \sim \mathrm{Uniform}(-1, 1)$ component-wise. The output is the
scalar response $y \in \mathbb{R}$.

\paragraph{Realised phase structure at the system seed.}
At the system seed used throughout
($\texttt{system\_seed} = 12$, $N_\phi = 4$, $\tau_G = 0.08$), all four
nominal phases carry appreciable mass on the simplex: a grid-evaluation
of $\pi(\phi \mid x_3)$ on $20{,}301$ uniform simplex points gives
marginal phase masses
$48.18\%$, $25.60\%$, $18.84\%$, $7.38\%$ (sorted, summing to $100\%$).
The boundary region defined by
$\max_\phi \pi(\phi \mid x_3) < 0.7$ occupies $11.40\%$ of the simplex
slice; tightening to $\max_\phi \pi < 0.5$ contracts it to $1.14\%$,
and relaxing to $\max_\phi \pi < 0.8$ expands it to $17.60\%$. The
realised problem is thus a genuine $4$-phase competition with curved
boundaries between phase domains in the $8$-dimensional input space.
Inside any single phase domain the property response collapses to a
single Gaussian whose mean varies smoothly with the $6$-dimensional
process subspace, while in the boundary region the conditional is a
genuine multi-component mixture.

\paragraph{Model architecture.}
Each ensemble member is an MDN with a $2$-hidden-layer MLP backbone,
$64$ hidden units per layer, GELU activations~\cite{hendrycks2016gelu},
and a final linear projection to $K_{\mathrm{MDN}} = 4$ mixture
components with diagonal covariances over the scalar property output.
The ensemble size is $n_{\mathrm{ens}} = 8$, with members initialised
from different PRNG seeds and trained independently on the same
labeled set.

\paragraph{Training schedule.}
All members are trained with AdamW (peak learning rate
$2\times 10^{-4}$, weight decay $5\times 10^{-2}$) preceded by adaptive
gradient clipping at $0.1$, with the same warmup-then-exponential-decay
schedule as in Appendix~\ref{app:exp_multimodal}. The number of gradient
steps per round is set adaptively to
$n_{\mathrm{iter}}(n_{\mathrm{lab}}) = \min(40{,}000,\;
200 \cdot n_{\mathrm{lab}})$ with a minimum of $2{,}000$ steps per
round, and mini-batch size $64$. The smaller peak LR and stronger
weight decay (relative to the multimodal and double-well benchmarks)
reflect the larger per-round step budget; as in the other benchmarks,
optimiser, architecture, and active-learning hyperparameters are fixed
across all acquisition functions and are not tuned per method.

\paragraph{Hyperparameters.}
Tables~\ref{tab:app_ternary_sim_hparams} and
\ref{tab:app_ternary_model_hparams} list the simulator/data and
model/AL hyperparameters used for all $30$ runs reported below.

\begin{table}[htbp]
\centering
\caption{Simulator and data hyperparameters for the ternary
phase-competition benchmark.}
\label{tab:app_ternary_sim_hparams}
\begin{tabular}{lll}
\toprule
Symbol / option & Value & Meaning \\
\midrule
$N_\phi$ & $4$ & Nominal phases (all four active at the chosen seed) \\
$\tau_G$ & $0.08$ & Free-energy softmin temperature \\
$n_{\mathrm{proc}}$ & $6$ & Process parameters per input \\
$c_{\mu\text{-scale}}$ & $6$ & Scale of per-phase composition coupling \\
$H_\phi$ raw scale & $0.5$ & $R_{\phi,ij} \sim \mathcal{N}(0, 0.5^2)$ \\
$H_\phi$ regulariser & $0.3 \cdot I_3$ & PD floor on Hessian \\
$b_\phi$ scale & $2$ & Linear-bias scale on composition \\
$d_\phi$ scale & $2$ & Per-phase mean offset scale \\
$\omega_\phi$ scale & $2$ & Sinusoidal-modulation frequency scale \\
$W_\phi$ scale & $1$ & Per-phase process-coupling scale \\
$e_\phi$ scale & $0.5$ & Composition-dependence of $\log \sigma_\phi^2$ \\
$f_\phi$ mean / std & $-1\,/\,0.3$ & Per-phase log-variance offset \\
$(x_A, x_B, x_C)$ prior & $\mathrm{Dirichlet}(1, 1, 1)$ & Composition prior on $3$-simplex \\
$p_i$ prior & $\mathrm{Uniform}(-1, 1)$ & Process-parameter prior \\
$\texttt{system\_seed}$ & $12$ & PRNG seed for $H_\phi, b_\phi, c_\phi, \dots$ \\
candidate pool size & $50{,}000$ & Unlabeled pool $|\mathcal{X}_{\mathrm{pool}}|$ \\
test set size & $2{,}000$ & Held-out evaluation set \\
initial labelled & $100$ & Initial labeled budget \\
\bottomrule
\end{tabular}
\end{table}

\begin{table}[htbp]
\centering
\caption{Model and active-learning hyperparameters for the ternary
phase-competition benchmark.}
\label{tab:app_ternary_model_hparams}
\begin{tabular}{lll}
\toprule
Symbol / option & Value & Meaning \\
\midrule
$n_{\mathrm{ens}}$ & $8$ & Ensemble size \\
$K_{\mathrm{MDN}}$ & $4$ & MDN mixture components per member \\
hidden features & $64$ & MLP width \\
depth & $2$ & Number of MLP hidden layers \\
activation & GELU & Nonlinearity~\cite{hendrycks2016gelu} \\
training batch size & $64$ & Mini-batch size for AdamW \\
peak learning rate & $2\times 10^{-4}$ & AdamW peak LR (after warmup) \\
weight decay & $5\times 10^{-2}$ & AdamW weight decay \\
gradient clip & $0.1$ & Adaptive gradient clipping threshold \\
min iter per round & $2{,}000$ & Floor on adaptive iteration count \\
$n_{\mathrm{iter}}$ (cap) & $40{,}000$ & Max gradient steps per round \\
iter-per-sample & $200$ & Slope of adaptive schedule \\
AL rounds & $30$ & Number of acquisition rounds \\
query batch size & $15$ & Queries acquired per round \\
acquisition batch size & $256$ & Chunk size for scoring the pool \\
seeds & $\{0, 1, 2, 3, 4\}$ & Data / training seeds \\
\bottomrule
\end{tabular}
\end{table}

\paragraph{Active-learning protocol.}
For each of $5$ seeds we instantiate a fresh pool of $50{,}000$
candidate inputs, a held-out test set of $2{,}000$ inputs, and an
initial labeled set of $100$ inputs. Acquisition scores are evaluated
on the remaining pool in chunks of $256$. Each run performs $30$
rounds of $15$ queries each, so the final labeled set contains
$100 + 30 \cdot 15 = 550$ inputs ($1.1\%$ of the pool). Test NLL is
evaluated at the end of every round.

\paragraph{Acquisition functions and selection strategies.}
We evaluate the same five base acquisitions as in Appendix~\ref{app:exp_multimodal} (Random, Epistemic Variance, MI-LB, BAIT~\cite{ash2021gone}, Core-Set~\cite{sener2018active}), paired with three selection strategies for the score-based ones: top-$k$ for Random / Variance / MI-LB, SBAL for Variance and MI-LB, and MaxDist for MI-LB. SBAL temperature is $T = 1.0$; MaxDist score weight is $w = 1$. BAIT and Core-Set bypass the score-to-batch step entirely (Fisher-trace forward+backward greedy and k-Center-Greedy on the MDN backbone respectively). Definitions follow Appendix~\ref{app:exp_multimodal} unchanged.

\paragraph{Seed-to-seed spread of the top-$k$ base acquisitions.}
Section~\ref{sec:exp_ternary_phases} reports that MI-LB is markedly more
consistent across seeds than Random or Epistemic Variance in the
data-scarce regime, with the gap closing by $n \gtrsim 400$.
Table~\ref{tab:app_ternary_seed_ranges} gives the underlying numbers:
the seed-to-seed min--max range of test NLL averaged over each budget
band. In the early band $n \in [115, 250]$, MI-LB's range is
$\sim\!2.6\times$ tighter than Random's and $\sim\!4.5\times$ tighter
than Epistemic Variance's; the inflated Variance spread comes from
seeds in which the second-moment signal selects unhelpful queries
before the ensemble has accumulated enough data to disagree
informatively over modal structure. By the late band ($n \gtrsim 400$)
all three ranges sit within a factor of $\sim\!2.4$ of each other.

\begin{table}[htbp]
\centering
\small
\caption{Seed-to-seed min--max range of test NLL (in nats), averaged
over each budget band, for the top-$k$ base acquisitions on the
synthetic phase-competition benchmark
($\texttt{system\_seed} = 12$, $5$ seeds).}
\label{tab:app_ternary_seed_ranges}
\begin{tabular}{lcc}
\toprule
Acquisition (top-$k$) & Early ($n \in [115, 250]$) & Late ($n \gtrsim 400$) \\
\midrule
Random             & $1.82$ & $0.24$ \\
Epistemic Variance & $3.13$ & $0.31$ \\
MI-LB                & $\mathbf{0.69}$ & $\mathbf{0.13}$ \\
\bottomrule
\end{tabular}
\end{table}

\paragraph{Results for SBAL and MaxDist variants.}
Table~\ref{tab:app_ternary_final_nll} reports the final test NLL at
$n = 550$ across $5$ seeds for all combinations, and
Figure~\ref{fig:app_ternary_sbal_maxdist} shows the corresponding
learning curves. Three observations.

First, MaxDist is the mean-best acquisition on this benchmark:
MI-LB (MaxDist) at $w = 1$ finishes at $1.945 \pm 0.061$, narrowly
ahead of MI-LB (top-$k$) at $1.985 \pm 0.043$. The seed-to-seed spread
is large enough that the two are not separated cleanly---per-seed,
MaxDist wins on $2$ seeds, ties on $1$, and is within $0.01$ nat of
top-$k$ on the other $2$---so the safest summary is that MaxDist
matches MI-LB (top-$k$) and offers a small mean improvement, the same
qualitative conclusion as on the coupled double-well benchmark of
Appendix~\ref{app:exp_coupled_double_well}.

Second, SBAL hurts MI-LB: MI-LB (SBAL) at $T = 1.0$ degrades from
$1.985 \pm 0.043$ to $2.109 \pm 0.054$, with MI-LB (top-$k$) winning on
all $5$ seeds. This is consistent with the small-budget,
narrow-high-information-manifold argument of
Appendix~\ref{app:exp_coupled_double_well}: at $T = 1.0$ enough mass
leaks onto low-score points that the effective fraction of
informative queries drops, and the realisable improvement from batch
diversity is bounded above by the MaxDist result that preserves the
informative ranking exactly.

Third, SBAL modestly helps Epistemic Variance:
$\mathbf{2.029 \pm 0.050}$ for SBAL versus $2.062 \pm 0.087$ for top-$k$, with
the SBAL variant winning on $4$ of $5$ seeds and producing a tighter
seed-to-seed spread. Because the Variance-based score is itself already
a noisier ranking signal than MI-LB on this benchmark, its top-$k$
ordering is less informative, and stochastic relaxation does not cost
what it does for MI-LB, at this temperature it actively helps by
hedging against unhelpful early queries.

\paragraph{BAIT and Core-Set.}
Core-Set ties MI-LB within seed-to-seed noise ($1.989 \pm 0.043$ vs $1.985 \pm 0.043$); the tie reflects benchmark geometry, with phase boundaries lying directly on the $2$-D composition simplex so that k-Center-Greedy in feature space hits the same boundary MI-LB targets through entropy disagreement. BAIT is the only acquisition whose mean NLL falls within seed-to-seed noise of Random ($2.190 \pm 0.090$ vs $2.209 \pm 0.085$): the $0.019$ improvement is an order of magnitude smaller than the $\sim\!0.09$ standard deviation of either run, whereas every other acquisition in Table~\ref{tab:app_ternary_final_nll} beats Random by at least $0.10$ in mean NLL --- well outside the noise band.
 The single-MC last-layer Fisher embedding produces a noisy candidate ranking on the $1$-D scalar output, and the $30 \times 15$ small-budget protocol does not give the forward+backward greedy enough rounds to recover --- the same failure mode anticipated in Appendix~\ref{app:exp_multimodal} (where BAIT had the largest seed-to-seed spread of any reported method) but more pronounced here because of the lower-dimensional output.

\begin{table}[htbp]
\centering
\small
\caption{Final test NLL at $n = 550$, mean $\pm$ std across $5$ seeds
$\{0,1,2,3,4\}$ (min--max in brackets) for the synthetic
phase-competition benchmark at $\texttt{system\_seed} = 12$. SBAL
temperature is $T = 1.0$ for both SBAL runs; MaxDist uses score weight
$w = 1$.}
\label{tab:app_ternary_final_nll}
\begin{tabular}{lcc}
\toprule
Acquisition & Mean $\pm$ std & Min -- Max \\
\midrule
Random                       & $2.209 \pm 0.085$ & $2.136$ -- $2.346$ \\
Epistemic Variance (top-$k$) & $2.062 \pm 0.087$ & $1.918$ -- $2.122$ \\
Epistemic Variance (SBAL)    & $2.029 \pm 0.050$ & $1.977$ -- $2.107$ \\
MI-LB (top-$k$)                & $1.985 \pm 0.043$ & $1.911$ -- $2.017$ \\
MI-LB (SBAL)                   & $2.109 \pm 0.054$ & $2.057$ -- $2.174$ \\
MI-LB (MaxDist)                & $\mathbf{1.945 \pm 0.061}$ & $1.870$ -- $2.023$ \\
BAIT                         & $2.190 \pm 0.090$ & $2.087$ -- $2.265$ \\
Core-Set                     & $1.989 \pm 0.043$ & $1.951$ -- $2.058$ \\
\bottomrule
\end{tabular}
\end{table}

\begin{figure}[htbp]
    \centering
    \includegraphics[width=0.85\textwidth]{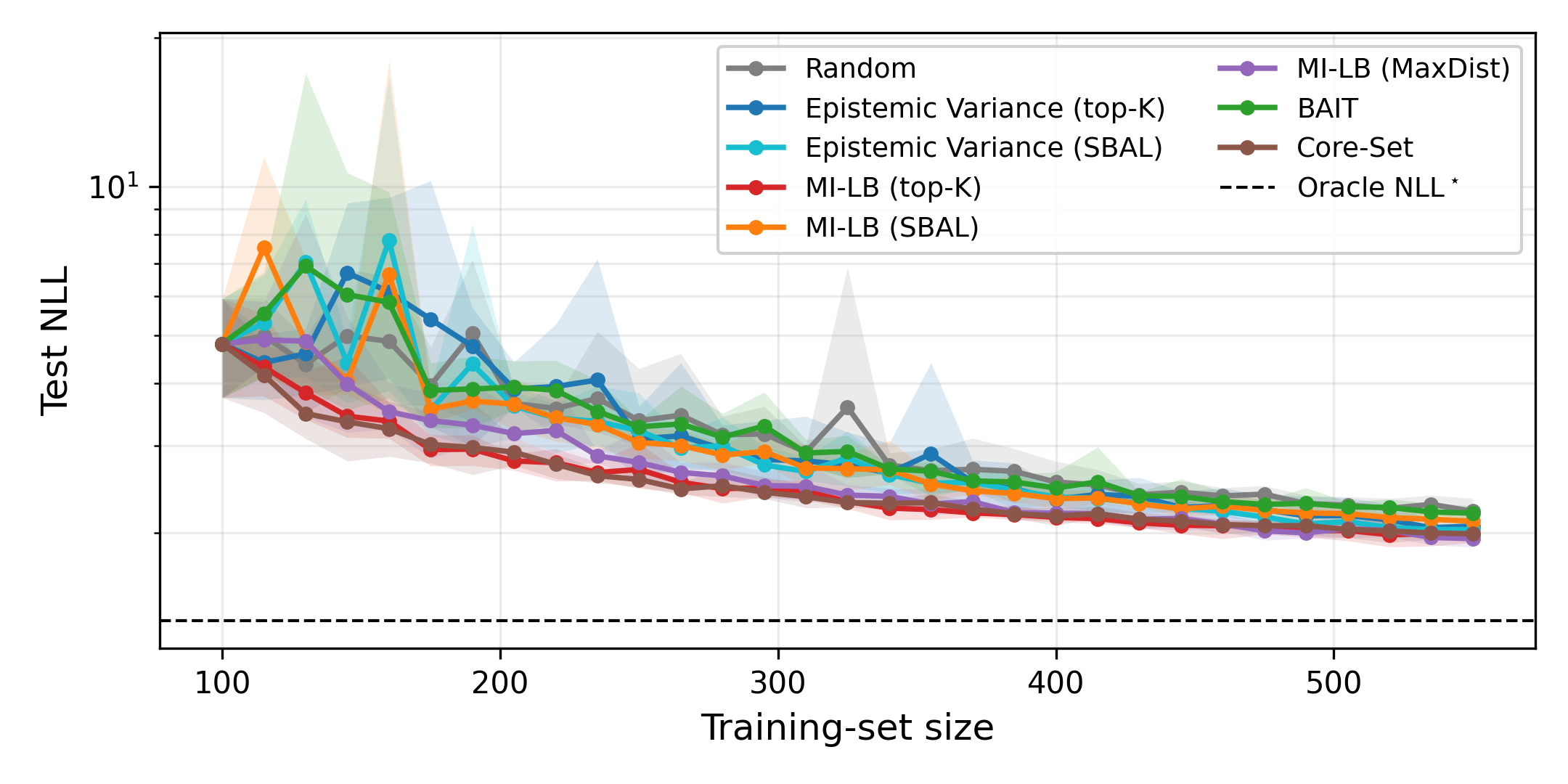}
    \caption{
        Synthetic phase-competition benchmark: test NLL vs.\ training-set size for all eight (acquisition, selection-strategy) combinations including the BAIT and Core-Set baselines ($5$ seeds; shaded bands min--max). SBAL $T=1.0$, MaxDist $w=1$. The dotted line marks the oracle floor $\mathrm{NLL}^{*} \approx 1.33$. The $y$-axis is clipped to $[1.2, 7.0]$ to suppress early-iteration training-instability spikes for the SBAL variants and BAIT.
    }
    \label{fig:app_ternary_sbal_maxdist}
\end{figure}

\paragraph{Remark on per-acquisition simplex visualisation.}
A per-acquisition visualisation of selected points projected onto the
$(x_A, x_B)$ simplex slice would complement the boundary-seeking
narrative of Section~\ref{sec:exp_ternary_phases}. The production runs
reported in Table~\ref{tab:app_ternary_final_nll} did not log query
indices to the experiment tracker, so we do not include such a figure
here; the example notebook accompanying this benchmark provides a
single-seed reproduction that materialises the simplex projection from
local state.

\paragraph{Why a mixture head: $K = 1$ vs $K = 4$.}
To justify the choice of a mixture head
($K_{\mathrm{MDN}} = 4$) over a single-Gaussian head ($K = 1$), we
train both with the identical architecture used in the AL experiments
(depth-$2$ MLP, $64$ hidden units, $n_{\mathrm{ens}} = 8$) on a larger
offline dataset of $100{,}000$ samples drawn from the same input
distribution as the active-learning experiments, and evaluate both on
a held-out test set of $10{,}000$ samples. The ensemble-averaged test
NLL is $\mathbf{1.34}$ for $K = 4$ versus $\mathbf{1.88}$ for
$K = 1$, against an analytically computed oracle floor of
$\mathrm{NLL}^{*} = 1.33$. The $K = 4$ ensemble therefore sits within
$0.01$ nat of the oracle---consistent with the fact that the simulator
is itself a $4$-component Gaussian mixture per input, so the $K = 4$
MDN model class contains the true distribution. The $K = 1$ model pays
a $0.55$-nat penalty at the same budget: forced to place a single
Gaussian per input, it cannot resolve phase competition in the
boundary region and must smear mass across phase means that differ by
several nats (per-phase median means $+2.9$, $+6.9$, $+3.7$, $-4.0$ on
this simulator instance). The penalty is smaller than on the coupled
double-well benchmark ($5.51$ nats, Appendix~\ref{app:exp_coupled_double_well})
because the per-phase response is a unimodal Gaussian with bounded
variance, so the smearing cost is bounded; it is still clearly
nonzero, and the active-learning experiments on this benchmark
therefore use $K_{\mathrm{MDN}} = 4$ throughout.

Figure~\ref{fig:app_ternary_mdn_baseline} shows the qualitative picture
on a $120$-point simplex grid with process parameters pinned to zero.
The $K = 4$ ensemble's predicted conditional mean
$\hat{\mathbb{E}}[Y \mid x]$ tracks the ground-truth
$\mathbb{E}^{*}[Y \mid x]$ to within roughly $\pm 1$ nat, with
residuals smoothly distributed across the simplex. The $K = 1$
residual is visibly structured and larger (up to $\pm 4$) at the
intersections of the phase-boundary network, where the single-Gaussian
head cannot simultaneously fit the correct mean and an appropriate
variance---exactly the failure mode anticipated in
Section~\ref{sec:exp_ternary_phases}.

\begin{figure}[htbp]
    \centering
    \includegraphics[width=\textwidth]{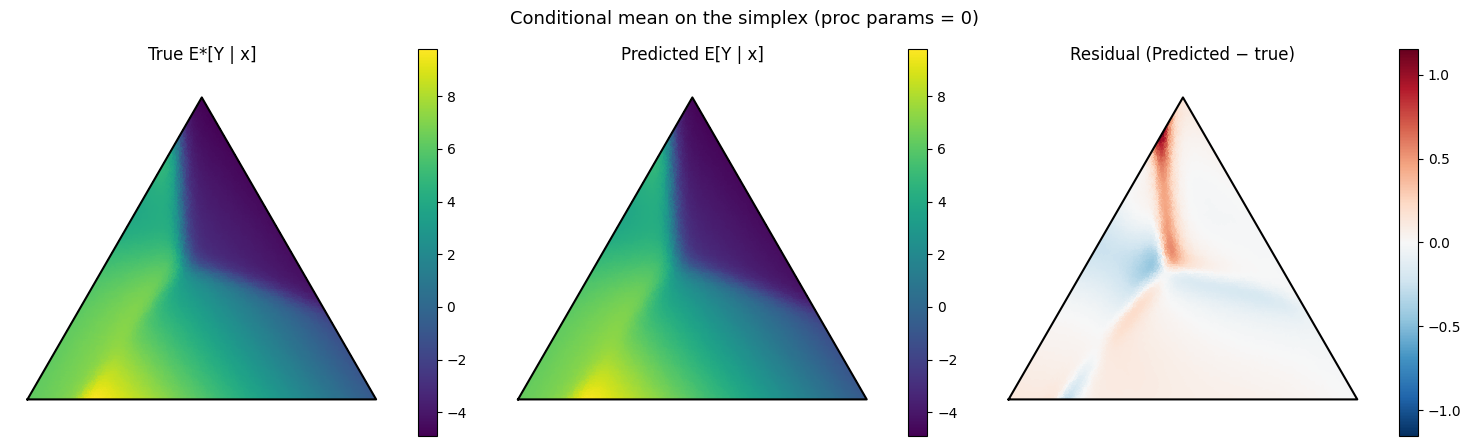}
    \\[1ex]
    \includegraphics[width=\textwidth]{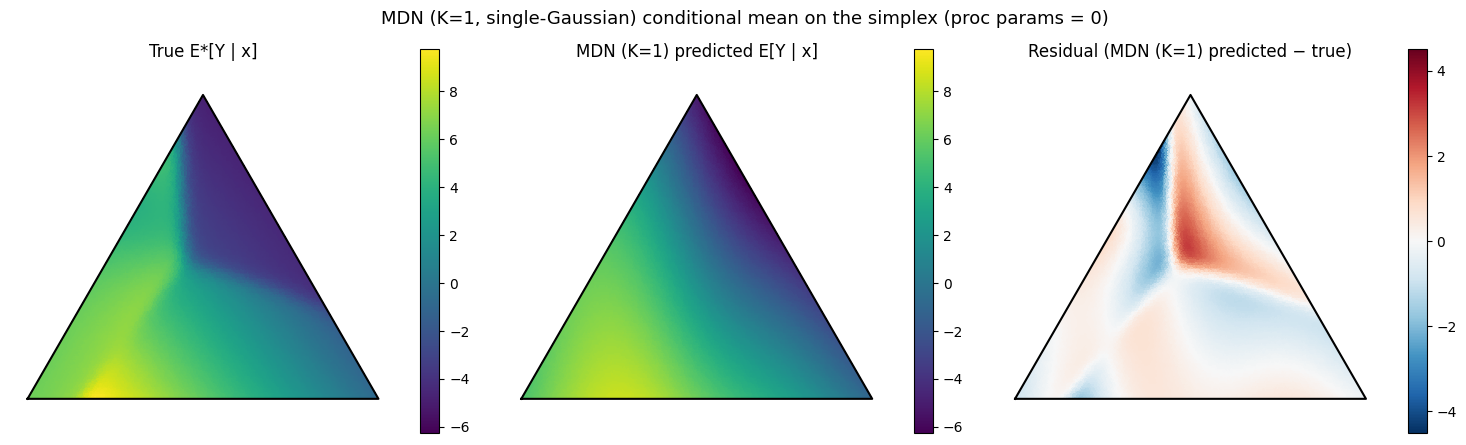}
    \caption{
    Synthetic phase-competition benchmark
    ($\texttt{system\_seed} = 12$): predicted conditional mean on a
    $120$-point simplex grid (process parameters pinned to zero) for
    the $K = 4$ MDN ensemble (\textbf{top row}) and a $K = 1$
    single-Gaussian MDN (\textbf{bottom row}), both trained offline on
    $100{,}000$ samples with the architecture used throughout the AL
    experiments.
    \textbf{Left column:} ground-truth $\mathbb{E}^{*}[Y \mid x]$.
    \textbf{Middle column:} ensemble-predicted
    $\hat{\mathbb{E}}[Y \mid x]$.
    \textbf{Right column:} residual (predicted $-$ true).
    The $K = 4$ residuals sit within $\sim\!\pm 1$ and are smoothly
    distributed; the $K = 1$ residuals reach $\pm 4$ and concentrate at
    the intersections of the phase-boundary network, where the
    single-Gaussian head cannot represent the competing phase
    components. The ensemble-averaged test NLL on a held-out set of
    $10{,}000$ samples is $1.34$ ($K = 4$) vs $1.88$ ($K = 1$),
    against an oracle floor of $1.33$.
    }
    \label{fig:app_ternary_mdn_baseline}
\end{figure}

\subsection{Compute Resources}
\label{app:compute_resources}

All experiments were run on a workstation with $7\times$ NVIDIA RTX
A6000 GPUs ($48$\,GB VRAM each), $128$ CPU cores, and approximately
$504$\,GiB of system memory. Each active-learning run uses a single
A6000 GPU; peak GPU memory fits comfortably within the $48$\,GB budget
for all configurations (small MLP backbones, ensemble size $8$), so
GPU memory is not a binding constraint. Per benchmark we evaluate
eight (acquisition, selection-strategy) combinations across $5$ seeds,
giving $40$ runs per benchmark and $120$ runs in total across the
three benchmarks. Per-run wall-clock times, extracted from the W\&B
logs and reported as the median over $5$ seeds, are summarised in
Table~\ref{tab:app_compute_runtimes}. Total compute across the three
benchmarks is approximately $25$ GPU-hours, of which $\sim\!12$\,h
(roughly half) is the $15$ BAIT runs,  BAIT is the only acquisition
with non-trivial selection overhead, since its forward$+$backward
greedy on the per-input Fisher embeddings of the full
$50{,}000$-point pool scales with both pool size and query batch. All
other acquisitions, including Core-Set, complete in roughly the same
wall time as Random per benchmark.

\begin{table}[htbp]
\centering
\small
\caption{Per-run wall-clock time on a single NVIDIA RTX A6000, median
over $5$ seeds, for the three paper benchmarks. Each entry is the
end-to-end runtime of one $20$-round (multimodal, double-well) or
$30$-round (ternary) active-learning run. Total compute across all
$120$ runs is $\sim\!25$ GPU-hours, of which $\sim\!12$\,h is BAIT
($\sim\!8.6$\,h on the double-well benchmark alone).}
\label{tab:app_compute_runtimes}
\begin{tabular}{lccc}
\toprule
Acquisition (selection) & Multimodal & Double-well & Ternary \\
\midrule
Random (top-$k$)               & $2.7$ min  & $3.1$ min  & $12.8$ min \\
Epistemic Variance (top-$k$)   & $3.2$ min  & $3.4$ min  & $13.6$ min \\
MI-LB (top-$k$)                & $3.1$ min  & $3.4$ min  & $13.1$ min \\
BAIT (top-$k$)                 & $22.4$ min & $1.35$ h   & $15.0$ min \\
Core-Set (top-$k$)             & $3.5$ min  & $4.4$ min  & $14.2$ min \\
Epistemic Variance (SBAL)      & $3.1$ min  & $3.4$ min  & $13.2$ min \\
MI-LB (SBAL)                   & $3.1$ min  & $3.4$ min  & $13.2$ min \\
MI-LB (MaxDist)                & $3.7$ min  & $3.9$ min  & $13.6$ min \\
\bottomrule
\end{tabular}
\end{table}



\end{document}